\documentclass[accepted]{uai2025} 
                        
\usepackage[american]{babel}

\usepackage{natbib} 

\usepackage{mathtools} 
\usepackage{booktabs} 
\usepackage{tikz} 

\usepackage[utf8]{inputenc} 
\usepackage[T1]{fontenc}    
\usepackage{hyperref}       
\usepackage{url}            
\usepackage{booktabs}       
\usepackage{amsfonts}       
\usepackage{nicefrac}       
\usepackage{amsthm}
\usepackage{bm}
\usepackage{microtype}      
\usepackage{xcolor}         
\usepackage{amsmath}
\usepackage{graphicx}
\usepackage{algorithm}
\usepackage{algorithmic}
\usepackage{wrapfig}
\usepackage{amssymb}
\usepackage{bm}
\usepackage{xspace}
\usepackage{cancel}
\usepackage{subcaption}
\usepackage{bbding}
\usepackage{mathrsfs}
\usepackage{etoc}
\usepackage{natbib}
\usepackage{hhline}
\usepackage{colortbl}
\usepackage{compact}
\usepackage{titlesec}
\usepackage{titletoc}

\DeclareMathOperator{\argmin}{arg\,min}

\newcommand{\ours}{\textsc{DF$^2$}\xspace}
\definecolor{green}{HTML}{009B55}
\definecolor{mygray}{gray}{0.9}

\newcommand{\ie}{\textit{i}.\textit{e}.,}
\newcommand{\eg}{\textit{e}.\textit{g}.}
\newtheorem{proposition}{Proposition}
\newcommand{\std}[1]{\tiny{$\pm$#1}}

\usepackage{xspace}

\title{\ours: Distribution-Free Decision-Focused Learning}

\author[1,2]{Lingkai Kong}
\author[2]{Wenhao Mu}
\author[2,3]{Jiaming Cui}
\author[2]{Yuchen Zhuang}
\author[2]{B. Aditya Prakash}
\author[2]{Bo Dai}
\author[2]{Chao Zhang}
\affil[1]{%
    Harvard University\\
    Cambridge, Massachusetts, USA
}
\affil[2]{%
    Georgia Institute of Technology\\
    Atlanta, Georgia, USA
}
\affil[3]{%
    Virginia Tech\\
    Blacksburg, Virginia, USA
  }
  
\begin{document}

\maketitle

\begin{abstract}
Decision-focused learning (DFL), which differentiates through the KKT conditions, has recently emerged as a powerful approach for predict-then-optimize problems. However, under probabilistic settings, DFL faces three major bottlenecks: model mismatch error, sample average approximation error, and gradient approximation error.  Model mismatch error stems from the misalignment between the model's parameterized predictive distribution and the true probability distribution. Sample average approximation error arises when using finite samples to approximate the expected optimization objective.  Gradient approximation error occurs when the objectives are non-convex and KKT conditions cannot be directly applied. In this paper, we present \ours—the first \textit{distribution-free} decision-focused learning method designed to mitigate these three bottlenecks. Rather than depending on a task-specific forecaster that requires precise model assumptions, our method directly learns the expected optimization function during training. To efficiently learn the function in a data-driven manner, we devise an attention-based model architecture inspired by the distribution-based parameterization of the expected objective. We evaluate \ours on two synthetic problems and three real-world problems, demonstrating the effectiveness of \ours. Our code can be found at: \url{https://github.com/Lingkai-Kong/DF2}.
\end{abstract}

\section{Introduction}

Many decision-making problems are fundamentally optimization problems that require the minimization of a cost function, which often depends on parameters that are both \textit{unknown} and \textit{context-dependent}. Typically, these parameters are estimated using observed features. For instance, hedge funds regularly recalibrate their portfolios to maximize expected returns, which involves predicting the future return rates of various stocks. Similarly, in personalized medicine, the selection of treatments for individual patients must predict unique responses to ensure optimal outcomes.



\begin{figure*}[t]
    \centering  \includegraphics[width=0.9\textwidth]{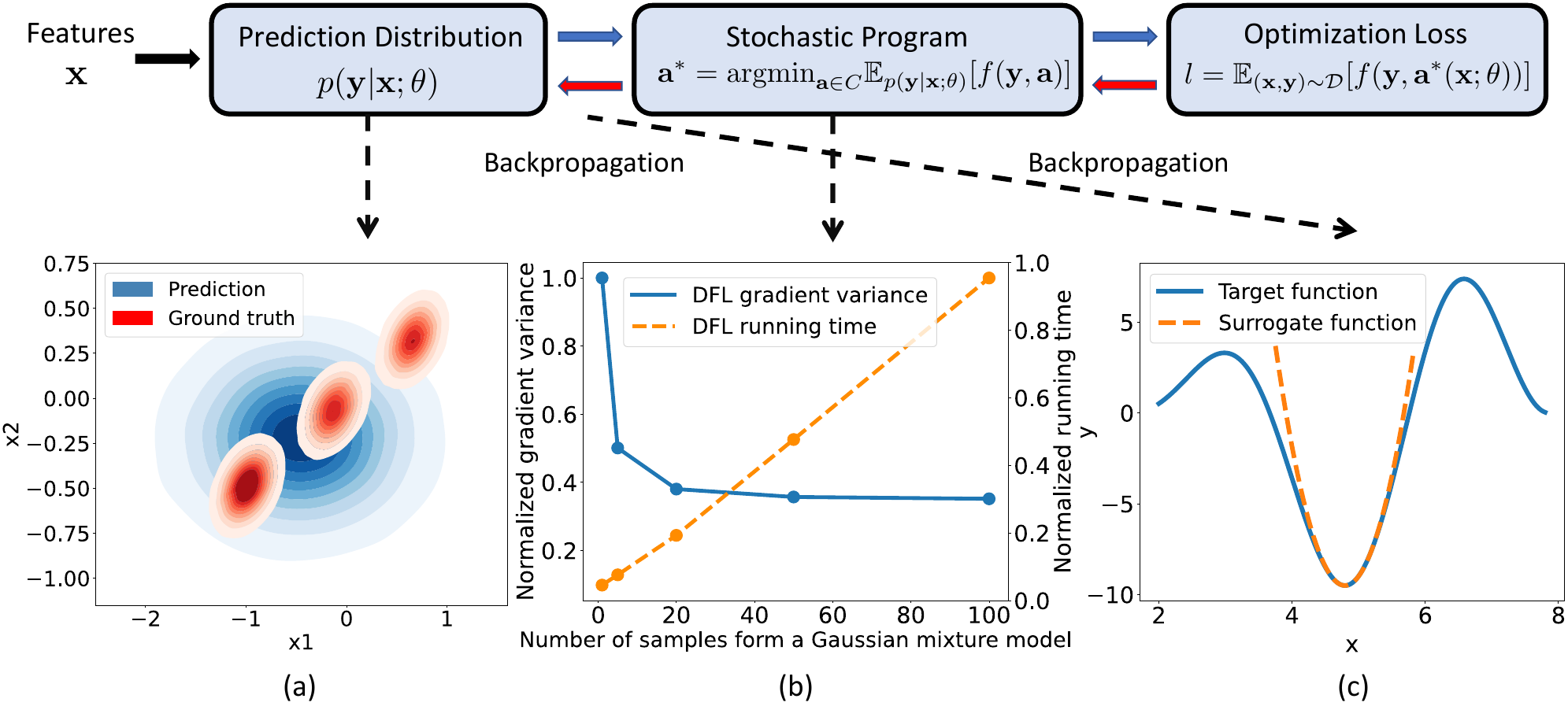}
    \caption{Decision-focused learning~\citep{donti2017task} directly optimizes the task loss and leads to better decision regret. However, it suffers from three significant bottlenecks. More illustrations are in Section \ref{sec:bottleneck}.}
    \label{fig:existing}
\end{figure*}

Given the growing capacity to train powerful deep learning models, a common strategy for this problem is the two-stage pipeline. This approach first learns a predictive model for unknown parameters using a generic loss function (\eg, negative log-likelihood) during the prediction stage, and then applies the model’s outputs in a downstream optimization problem. Despite its widespread use, this pipeline implicitly assumes that better predictive accuracy—measured by the prediction loss—translates to better optimization performance. However, this assumption often breaks down, as prediction errors can have non-uniform effects on the optimization objective.
To address this issue, \emph{Decision-Focused Learning (DFL)}~\citep{donti2017task, wilder2019melding,  wang2020automatically, sun2023alternating, yan2021surrogate, rodriguez2024right} integrates the prediction and optimization stages into a single end-to-end model. A prominent line of work leverages the implicit function theorem and the KKT conditions to differentiate through the optimization layer~\citep{donti2017task}, enabling the learning process to align predictive outputs directly with decision quality. This results in models that are trained explicitly for decision-making, often framed as regret minimization.

Despite its promising results, DFL via differentiation through KKT conditions faces several critical bottlenecks in the probabilistic setting, where the predictive model outputs a distribution rather than a point estimate:
(1) Model mismatch error: real-world applications often operate in highly uncertain environments and involve complex, multimodal probability distributions. In contrast, DFL by differentiating through KKT conditions requires simple parameterized distribution models for computational feasibility, leading to a mismatch. (2) Sample average approximation error: When there is no closed-form expression for the expected optimization objective, we typically draw a finite number of samples from the distribution for averaging, which will introduce extra statistical errors. (3) Gradient approximation error:  the KKT condition is only a sufficient condition for optimal solution of convex problem, which is unable to characterize the optimal solution in non-convex setting, and thus, will lead to inaccuracies that cumulatively result in lower decision quality. Recent works \citep{kongend, shah2022learning, shah2023leaving} have proposed surrogate objectives to bypass the challenges of gradient computation. However, these approaches are still model-based and suffer from the other two bottlenecks. While SPO~\citep{elmachtoub2022smart} generally converges to a decision with optimal expected costs regardless of the distribution, it is restricted to linear objectives.


We propose \ours, the first distribution-free decision-focused learning method, to mitigate the three bottlenecks and handle complex objectives beyond the linear class.  Instead of relying on a task-specific forecaster that necessitates precise model assumptions, we propose to learn directly the expectation of the optimization objective function from the data. Upon learning, we can obtain the optimal decision by maximizing the learned expected function within the feasible space. In order to ensure that the network architecture lies within the true model class and minimize bias error, we have developed an attention-based network architecture that emulates the distribution-based parameterization of the expected objective. This attention architecture also preserves the convexity of the original optimization objective.  In contrast to the two-stage model, \textit{\ours} is decision-aware. Compared to DFL methods, \textit{\ours} avoids model mismatch error, gradient approximation error, and sample average approximation error at test time.

Our main contributions can be summarized as follows: (1) We propose a distribution-free training objective for DFL. It mitigates the three bottlenecks of existing methods under the probabilistic setting. (2) We propose an attention-based network architecture inspired by the distribution-based parameterization to ensure the network architecture is within the true model class. (3) Experiments on two synthetic datasets and three real-world datasets show that our method can achieve better performance than existing DFL methods.

\section{Preliminaries}
\label{sec:bottleneck}

\subsection{Decision-Focused Learning by Differentiating Through KKT Conditions}

In the predict-then-optimize problem, a predictor $\mathbf{M}_{\theta}$
  inputs features $\mathbf{x}$ and outputs a point estimate $\mathbf{\hat y}$. This estimate parameterizes the optimization problem $\argmin_{\mathbf{a}\in C} f(\mathbf{y}, \mathbf{a})$, where $f$ is the cost function, $\mathbf{a}$ is the decision variable, and $C$ is the feasible space.

However, point estimations fail to capture the uncertainty inherent in model predictions \citep{abdar2021review} and the stochastic nature of the problem parameters \citep{schneider2007stochastic}. To address this, we focus on a probabilistic framework, wherein the predictor's output is a probability distribution $p_{\theta}(\mathbf{y}|\mathbf{x})$, rather than a mere point estimate. This allows us to engage in stochastic optimization, where the objective is to find the optimal action $\mathbf{a}^*(\mathbf{x};\theta)$ that minimizes the expected cost, formalized as $\argmin_{\mathbf{a}\in C}\mathbb{E}_{p_{\theta}(\mathbf{y}|\mathbf{x})}[f(\mathbf{y},\mathbf{a})]$. This method more effectively accounts for the uncertainties and variabilities present in the parameters.

Predictions are  then evaluated  based on the decision loss they generate, essentially the cost function's value using the true parameters $\mathbf{y}$. For a dataset $\mathcal{D}=\{\mathbf{x}_i,\mathbf{y}_i \}_{i=1}^N$, the goal is to train a  model \(\mathbf{M}_{\boldsymbol{\theta}}\) to minimize the decision loss:
\begin{align}  \textstyle \theta^*=\argmin_{\theta}\frac{1}{N}\sum_{i=1}^Nf(\mathbf{y}_i, \mathbf{a}^*(\mathbf{x}_i;\theta)).
\end{align}
By directly optimizing the decision loss, the gradient of the model parameters can be calculated using the chain rule: $\frac{\mathrm{d} f(\mathbf{y},\mathbf{a}^{*}(\mathbf{x};\theta))}{\mathrm{d} \theta}=\frac{\mathrm{d} f(\mathbf{y},\mathbf{a}^{*}(\mathbf{x};\theta))}{\mathrm{d} \mathbf{a}^{*}(\mathbf{x};\theta)}\frac{\mathrm{d} \mathbf{a}^{*}(\mathbf{x};\theta)}{\mathrm{d} {\theta}}.$ To compute the Jacobian $\frac{\mathrm{d} {\mathbf{a}^{*}(\mathbf{x};\theta)}}{\mathrm{d} \theta}$ for backpropagation, OptNet \citep{amos2017optnet}  assumes quadratic optimization objectives and differentiates through the KKT conditions using the implicit function theorem. Later, cvxpylayers \citep{agrawal2019differentiable} extends it to more general cases of convex optimization using disciplined parameterized programming (DPP) grammar.


\begin{figure*}
    \centering
    \includegraphics[width=0.9\textwidth]{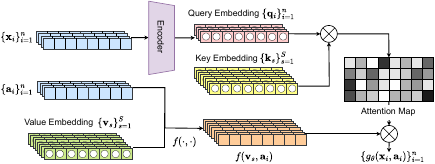}
    \caption{The proposed attention-based network architecture of \ours. The network contains an encoder and a set of learnable attention points $\{\mathbf{k}_s,\mathbf{v}_s\}_{s+1}^S$. Given an input feature $\mathbf{x}$, the encoder first project it to query embedding space and then compute the attention weights by its dot product with the key embeddings. The final function value $g(\mathbf{x},\mathbf{a})$ is a weighted combination of $f(\mathbf{v},\mathbf{a})$.  The designed network architecture can effectively reduce the bias error in Proposition~\ref{prop:1}.}
    \label{fig:overall}
\end{figure*}

\subsection{Bottlenecks under the Probabilistic Setting}

Although DFL by differentiating through KKT condition can achieve better decisions compared to the two-stage learning, they have three significant bottlenecks under the probabilistic setting.

\noindent\textbf{Bottleneck 1: Model Mismatch Error.} Real-world applications often involve complex and multi-modal probability distributions $p(\mathbf{y}|\mathbf{x})$~\citep{kong2023uncertainty, kong2024diffusion, kong2024two,li2023muben}. One prominent example is the wind power forecasting task, where the environment exhibits high uncertainty due to the dynamic and stochastic nature of wind patterns. Factors such as changing weather conditions, terrain, and turbulence can significantly affect the true distribution of wind power, making it highly intricate and challenging to model accurately.

However, existing approaches \citep{donti2017task, kongend} tend to assume simple distributions, \eg, isotropic Gaussian distribution, for computational feasibility. However, this assumption can lead to considerable misalignment between the model's parameterized distribution and the true underlying distribution in tasks with high uncertainty. This mismatch results in poor approximations and reduced decision-focused learning performance. Fig.~\ref{fig:existing}(a) illustrates this issue using a ground-truth distribution composed of a mixture of three Gaussians. As we can see, the performances of DFL approaches suffer due to the model mismatch error, which is particularly pronounced in tasks with highly uncertain environments.

\noindent\textbf{Bottleneck 2: Sample Average Approximation Error.} In complex optimization problems, closed-form expressions for expectations might be unavailable, necessitating the use of sample average approximation \citep{kim2015guide,verweij2003sample,kleywegt2002sample}. Although adopting a more expressive distribution, such as a mixture density network, could potentially improve performance, doing so introduces another issue—sample approximation error. As shown in Fig.~\ref{fig:existing}(b), when dealing with intricate distributions, increasing the sample size reduces the gradient variance slowly but demands substantially higher computational resources and longer running times. 

\noindent\textbf{Bottleneck 3: Gradient Approximation Error.} The KKT condition can only be applied to convex objectives. However, many real-world applications involve complicated non-convex objectives. Though \citet{perrault2019decision, wang2020scalable} propose to approximate the non-convex objectives by a quadratic function around a local minimum to approximate $\frac{\mathrm{d} \mathbf{a}^*}{\mathrm{d} \theta}$ (Fig.~\ref{fig:existing}(c)), the inaccurate gradients may be aggregated during the training iterations and thus lead to poor decisions. 

The first two errors occur during both training and testing, whereas gradient approximation errors occur only during training. Recently, several methods~\citep{kongend,shah2022learning, shah2023leaving} have proposed surrogate losses for DFL to avoid differentiating through KKT conditions. However, they still suffer from the first two bottlenecks.

It should be noted that when the objective function is linear, the expectation of a linear function has a closed-form expression and only requires estimating the mean of a distribution. Therefore, the model does not suffer from these bottlenecks. As a result, SPO \citep{elmachtoub2022smart} proves that it converges to a decision with optimal expected costs regardless of the distribution. In our paper, we consider a more complex setting where estimating the expected cost requires the entire predictive distribution.


\section{Distribution-Free Decision-Focused Learning}
In this section, we introduce \ours which mitigates all the three bottlenecks within a single model. We first introduce the distribution-free training objective which transforms DFL into a function approximation problem. Then, we design an attention-based architecture inspired by the distribution-based parameterization to reduce the bias error. Finally, we discuss how to obtain the optimal decision during inference.

\subsection{Distribution-Free Training Objective}
Existing DFL methods primarily rely on a distribution-based approach. These techniques learn a forecaster that outputs probability distribution $p(\mathbf{y}|\mathbf{x})$ based on various model assumptions. However, a more straightforward approach is to estimate the expected cost function $\mathbb{E}_{p(\mathbf{y}|\mathbf{x})}[f(\mathbf{y},\mathbf{a})]$ directly from the training data $\mathcal{D}=\{{\mathbf{x}_i, \mathbf{y}_i\}}_{i=1}^N$.

The cornerstone of our method is the observation that the expected cost objective is only a function of $\mathbf{a}$ and $\mathbf{x}$, which is represented as $g(\mathbf{x}, \mathbf{a})=\mathbb{E}_{p(\mathbf{y}|\mathbf{x})}[f(\mathbf{y},\mathbf{a})]$. We propose a direct approach to learn a neural network with parameters $\theta$ to match the expected cost function $\mathbb{E}_{p(\mathbf{y}|\mathbf{x})}  [f(\mathbf{y},\mathbf{a})]$. Our objective is to minimize the mean square error (MSE) between the fitted function $g(\mathbf{x},\mathbf{a})$ and the cost function $f(\mathbf{y},\mathbf{a})$ sampled from $p(\mathbf{x}, \mathbf{y})$:
\begin{align}
g^*(\mathbf{x},\mathbf{a})= \argmin_{g} \mathbb{E}_{(\mathbf{x},\mathbf{y})} \mathbb{E}_a[g(\mathbf{x},\mathbf{a})-f(\mathbf{y},\mathbf{a})]^2.
\label{eq:objective}
\end{align}
The proposed training objective can be efficiently optimized using stochastic gradient-based methods such as ADAM \citep{kingma2015adam}. 

In the ideal case, when we have infinite training data and model capacity, the optimal solution $g^*$ of Eq.~\ref{eq:objective} is the ground-truth conditional expectation $\mathbb{E}_{p(\mathbf{y}|\mathbf{x})}[f(\mathbf{y},\mathbf{a})]$. Upon learning the optimal function, the optimal action can be derived by maximizing the fitted function $\mathbf{a}^*=\argmin_{\mathbf{a}\in C} g_{\theta}(\mathbf{x}, \mathbf{a})$. However, in practical situations where training data and model capacity are limited, we obtain the expected error on the test set as the following proposition.
\begin{proposition}\label{prop:1}
The expected MSE of the optimal solution $g^*$ on the test set is: 
\begin{align}
 \text{MSE}_{\rm test}  &=  \underbrace{\mathbb{E}_{\mathcal{D}'}\left [ \left(g^*_{\mathcal{D}'}(\mathbf{x,\mathbf{a}})- \mathbb{E}_{p(\mathbf{y}|\mathbf{x})}[f(\mathbf{y},\mathbf{a})]   \right)^2 \right]}_{\text{Bias}}\nonumber
 \\ &+ \underbrace{\mathbb{E}_{\mathcal{D}'}\left[ \left(g^*_{\mathcal{D}'}(\mathbf{x,\mathbf{a}})-  \mathbb{E}_{\mathcal{D}'}[g^*_{\mathcal{D}'}(\mathbf{x,\mathbf{a}})]  \right)^2 \right]}_{\text{Variance}}, \nonumber
\end{align}
where $\mathcal{D}'$ denotes the training dataset  augmented with the sampled actions $\mathbf{a}$, and $g^*_{\mathcal{D}'}(\mathbf{x},\mathbf{a})$ denotes the function fitted on the dataset $\mathcal{D}'$.

Proof. See Appendix~\ref{s:Prop1} for a detailed proof.
\end{proposition}

\noindent \textbf{Sampling Action from the Constrained Space}. In practice, it's unnecessary to fit the true objective across the entire Euclidean space. Instead, we only need to sample from the constrained space $C$. There are several strategies for this. One simple approach is to sample from a relaxed version of the constrained space,such as an outer bounding box that encloses $C$. This allows us to sample each dimension of $\mathbf{a}$ independently from a uniform distribution. Moreover, many predict-then-optimize problems are resource allocation problems where the decision variable $\mathbf{a}$ is a simplex; for a simplex, we can directly sample from the Dirichlet distribution. Appendix~\ref{s:sampling} provides more illustrations on the relaxed constrained sampling. Alternatively, we can employ Markov chain Monte Carlo (MCMC) methods to uniformly sample within $C$, such as Ball Walk~\citep{lovasz1990mixing} and the hit-and-run algorithm~\citep{belisle1993hit, lovasz1999hit}. However, these methods typically incur higher computational costs.

In contrast to traditional DFL, our framework effectively transforms decision-focused learning into a function approximation problem, circumventing the complexities of solving and differentiating through the optimization problem. This approach avoids both model mismatch error and gradient approximation error. While we do not claim to fully address the sample average approximation error during training, as we still rely on finite data to estimate the expected cost function, we can avoid this error at inference time, see Section~\ref{sec:inference}. 


As we can see from Proposition~\ref{prop:1}, the test MSE consists of the bias and variance terms. The variance term will be reduced by sampling more data. To ensure that the bias error term approaches zero with more training data, it is crucial to keep the network architecture within the model class. To tackle this challenge, we introduce an attention-based network architecture in the following subsection. 











\subsection{Distribution-Based Parameterization}
The key of our architecture design is to mimic the distribution-based parameterization of the expected cost function. Since our training objective bypass the need of    solving and differentiating through the stochastic optimization problem, we can adopt an expressive non-parametric distribution with kernel conditional mean embedding (CME) to parameterize our model. The proposed network  architecture can lead to zero bias error in Proposition~\ref{prop:1}

\begin{table}
\small
\centering
\begin{tabular}{c c c}
\toprule[1.5pt]
Variable & $\mathbf{x}$ & $\mathbf{y}$ \\
Domain & $\mathcal{X}$ & $\mathcal{Y}$ \\
Kernel & $\mathcal{R}_{\mathbf{x}}(\mathbf{x},\mathbf{x}')$ & $\mathcal{R}_{\mathbf{y}}(\mathbf{y},\mathbf{y}')$ \\
Feature map & $\mathcal{R}_{\mathbf{x}}(\mathbf{x},\cdot)$ & $\mathcal{R}_{\mathbf{y}}(\mathbf{y},\cdot)$ \\
RKHS & $\mathcal{G}$ & $\mathcal{F}$ \\
\bottomrule[1.5pt]
\end{tabular}
\caption{Table of Notations}
\label{table:notations}
\end{table}

CME \citep{song2009hilbert, song2013kernel} is a powerful tool to compute the expectation of a function in the reproducing kernel Hilbert space (RKHS), without the curse of dimensionality. Let $\mathcal{F}$ be a RKHS over the domain of $\mathbf{y}$ with kernel function $\mathcal{R}_\mathbf{y}(\mathbf{y},\mathbf{y}')$ and inner product $\langle \cdot, \cdot \rangle_{\mathcal{F}}$. 
For a particular $\mathbf{a}$, we denote the corresponding function as $f_\mathbf{a}(\mathbf{y})$. CME projects the conditional distribution to its expected feature map $\mu_{\mathbf{y}|\mathbf{x}}\triangleq \mathbb{E}_{p(\mathbf{y}|\mathbf{x})}[\mathcal{R}_{\mathbf{y}}(\mathbf{y}, \cdot)]$ and evaluates 
the conditional expectation of any RKHS function, $f_\mathbf{a} \in \mathcal{F}$, as an inner product in $\mathcal{F}$ using the reproducing property:
\begin{align}
\mathbb{E}_{p(\mathbf{y}|\mathbf{x})}[f_{\mathbf{a}}]&=\int p(\mathbf{y}|\mathbf{x})\langle \mathcal{R}_y(\mathbf{y},\cdot), f_{\mathbf{a}} \rangle_{\mathcal{F}}d\mathbf{y} \nonumber \\
&= \left\langle \int p(\mathbf{y}|\mathbf{x})\mathcal{R}_y(\mathbf{y}, \cdot)\mathrm{d}\mathbf{y}, f_{\mathbf{a}} \right\rangle_{\mathcal{F}} =   \langle \mu_{\mathbf{y}|\mathbf{x}}, f_{\mathbf{a}}\rangle_{\mathcal{F}}. \nonumber
\end{align}
Assume that for all $f_{\mathbf{a}}\in \mathcal{F}$, the conditional expectation $\mathbb{E}_{p(\mathbf{y}|\mathbf{x})}[f_{\mathbf{a}}(\mathbf{y})]$ is an element of the RKHS over the domain of $\mathbf{x}$, the conditional embedding can be estimated with a finite dataset $\{\mathbf{x}_s, \mathbf{y}_s\}_{s=1}^S$ as
$\hat{\mu}_{\mathbf{y}|\mathbf{x}}=\sum_{s=1}^S\beta_s(\mathbf{x})\mathcal{R}_{\mathbf{y}}(\mathbf{y}_s,\cdot)$,
 where $\beta_s$ is a real-valued weight and can be computed with matrix calculation (see more details about this computation in Appendix~\ref{s:Background}). 
 
 One advantage of CME is that $\hat{\mu}_{\mathbf{y}|\mathbf{x}}$ can converge to $\mu_{\mathbf{y}|\mathbf{x}}$ in the RKHS norm at an overall rate of $\mathcal{O}(S^{-\frac{1}{2}})$ \citep{song2009hilbert}, which is independent of the input dimensions. This property let CME works well in the high-dimensional space.  With the estimated CME, the conditional expectation can be computed by the reproducing property:
 \begin{align}
\mathbb{E}_{p(\mathbf{y}|\mathbf{x})}[f_{\mathbf{a}}(\mathbf{y})] &=\langle \hat{\mu}_{\mathbf{y}|\mathbf{x}}, f_{\mathbf{a}}\rangle_{\mathcal{F}} = \left\langle  \sum_{s=1}^S\beta_s(\mathbf{x})\mathcal{R}_{\mathbf{y}}(\mathbf{y}_s, \cdot), f_{\mathbf{a}} \right\rangle_{\mathcal{F}}
\nonumber\\& = \sum_{s=1}^S\beta_s(\mathbf{x})f_{\mathbf{a}}(\mathbf{y}_s).
\label{eq:cme}
\end{align}
As shown in Eq.~\ref{eq:cme}, the formulation is essentially a weighted combination of $f_\mathbf{a}(\mathbf{y}_s)$, where the weights are conditioned on the input features $\mathbf{x}$. This observation inspires us to leverage attention-based parameterization to represent the function $g(\mathbf{x},\mathbf{a})$. The attention mechanism forms the foundation of the transformer architecture \citep{vaswani2017attention} and has been successfully utilized across various deep learning applications \citep{kenton2019bert,brown2020language,dosovitskiy2021an}. 

Inspired by this, we introduce a set of learnable attention points $\{{\mathbf{k}_s, \mathbf{v}_s}\}_{s=1}^S$, where $\mathbf{k}$ is the key embedding and $\mathbf{v}$ is the corresponding value embedding. For an input $\mathbf{x}$, the encoder first maps it to the query embedding space $\mathbf{q}$ and compute the attention weights by its product with the key embeddings. We set the value function as $f(\mathbf{v}_s,\mathbf{a})$ and, consequently, reformulate the function $g(\mathbf{x},\mathbf{a})$ using the softmax attention mechanism \citep{vaswani2017attention}:
\begin{align}
g(\mathbf{x},\mathbf{a})= &\text{Softmax}\left(\left[\frac{\mathbf{q}(\mathbf{x})^\top\mathbf{k}_1}{\sqrt{d}}, \cdots, \frac{\mathbf{q}(\mathbf{x})^\top\mathbf{k}_S}{\sqrt{d}}\right]\right)^\top \nonumber\\&[f(\mathbf{v}_1, \mathbf{a}), \cdots, f(\mathbf{v}_S, \mathbf{a})],
\label{eq:attention}
\end{align}
where $d$ is the dimension size of the key embeddings and value embeddings. 
\begin{proposition}\label{prop:my_proposition}
It holds for any $\mathbf{x}$ and $\mathbf{a}$, the function $g(\mathbf{x},\mathbf{a})$ defined by the softmax attention in Eq.~\ref{eq:attention} $\mathbb{E}_{\hat{p}_{\mathcal{R}}(\mathbf{y}|\mathbf{x})}[f(\mathbf{y},\mathbf{a})]=g(\mathbf{x},\mathbf{a})$. Here,
$\hat{p}_{\mathcal{R}}(\mathbf{y}|\mathbf{x})$ is a  parameterization restriction of $p(\mathbf{y}|\mathbf{x})$. 

Proof. See Appendix~\ref{s:Prop2} for a detailed proof.
\label{prop:2}
\end{proposition}       
From Proposition \ref{prop:2}, it is evident that with the attention-based network architecture, we can guarantee that our learned expected function resides within the true model class. To speed up the training procedure, one can initialize the value embeddings of the attention points with randomly selected labels from the training dataset. This approach provides a reasonable starting point for the model and reduces the time it takes for the model to converge to a solution. The training procedure of \ours is given in Algorithm~\ref{alg:training}.

\noindent\textbf{Remark.} Our proposed attention-based network architecture represents a parameterization of $p(\mathbf{y}|\mathbf{x})$, drawing similarities with the two-stage model and DFL. Compared with the two-stage model, we learn the expected cost function to make \ours decision-aware. Compared with DFL, we do not have to solve the stochastic optimization problem during learning. As a result, we can adopt an expressive nonparametric distribution with CME to parameterize $p(\mathbf{y}|\mathbf{x})$ to avoid the model mismatch error.

\subsection{Model Inference}
\label{sec:inference}

At test time, we can obtain the optimal decision by maximizing the learned expected cost $\argmin_{\mathbf{a}\in C}g(\mathbf{x},\mathbf{a})$. The final representation of $g(\mathbf{x}, \mathbf{a})$ is a weighted combination of $f(\mathbf{v}_s, \mathbf{a})$ with different value embeddings. Another benefit of the proposed attention-based network architecture is that it can preserve the convex property of the cost function.
\begin{proposition}
    As long as $f(\mathbf{y},\mathbf{a})$ is a convex function with respect to $\mathbf{a}$, $g(\mathbf{x},\mathbf{a})$ is a convex function with respect to $\mathbf{a}$.
    
Proof: This is a direct consequence of the theorem that a convex combination of convex functions remains a convex function
\end{proposition}

When the original objective is convex, we can use any existing black-box convex solver \citep{diamond2016cvxpy, agrawal2018rewriting, gurobi}. For non-convex problem, we can use projected gradient descent.


Although Eq.~\ref{eq:objective} involves sampling \(\mathbf{x}, \mathbf{y}\) during training, this introduces generalization error due to the finite size of the training dataset. Crucially, this generalization error is distinct from the SAA error, which arises in existing methods that require sampling from a predicted distribution (e.g., a Gaussian with learned parameters) to estimate an expected objective. In such cases, the generalization error in the predictive model is further compounded by the additional variance introduced through sampling, resulting in compounded inaccuracies.

In contrast, our method learns the expected objective \(g(\mathbf{x}, \mathbf{a})\) directly and does not require sampling at inference time, thereby eliminating the additional SAA error. Nonetheless, like all learning-based methods, it remains subject to generalization error stemming from limited training data.

\section{Additional Related Work}
SO-EBM \citep{kongend} proposes a surrogate learning objective by maximizing the likelihood of the pre-computed optimal decision within an energy-based probability parameterization. LODL \citep{shah2022learning, shah2023leaving} and LANCER~\citep{zharmagambetov2023landscape}  approximate the decision-focused loss with a quadratic function or a neural network. \ours is different from them:  (1) They assume a deterministic setting while we assume the problem parameter $\mathbf{y}$ is a probability distribution.
(2) They approximate the decision loss which is a function of the problem parameter $\mathbf{y}$. In contrast,
\ours directly learns the expected cost function which remains independent of $\mathbf{y}$. (3) They still relies on initially learning a forecaster to infer $\mathbf{y}$ from $\mathbf{x}$. Consequently, they remain susceptible to both model mismatch error and sample average approximation error in our probabilistic setting.  Recently, \citet{bansal2023taskmet} proposes TaskMet with the motivation to simultaneously optimize predictive loss and decision loss, rather than addressing the three bottlenecks.

Several other works have focused on linear objectives, where DFL through KKT condition may encounter singular value issues. To address this, the SPO+ loss \citep{elmachtoub2022smart} evaluates prediction errors relative to optimization objectives using the subgradient method. The approach by \citet{wilder2019melding} incorporates a quadratic regularization term for smoothing. Meanwhile, \citet{mandi2020interior} introduces a log barrier regularizer and differentiates through the homogeneous self-dual embedding. In contrast, our method is crafted for a broader range of objectives.

When the optimization problem is discrete, differentiating through the optimization layer is even more challenging since the gradient is ill-defined in the discrete domain. Various solutions have been proposed, such as tackling the discrete challenge via interpolation \citep{poganvcic2019differentiation}, perturbation \citep{niepert2021implicit, berthet2020learning}, subgradient methods \citep{mandi2020smart}, and cutting planes \citep{ferber2020mipaal}. Our method is directly applicable to the discrete setting and we leave it for future exploration.


\begin{figure*}[!t]
    \centering
    \includegraphics[width=0.9\textwidth]{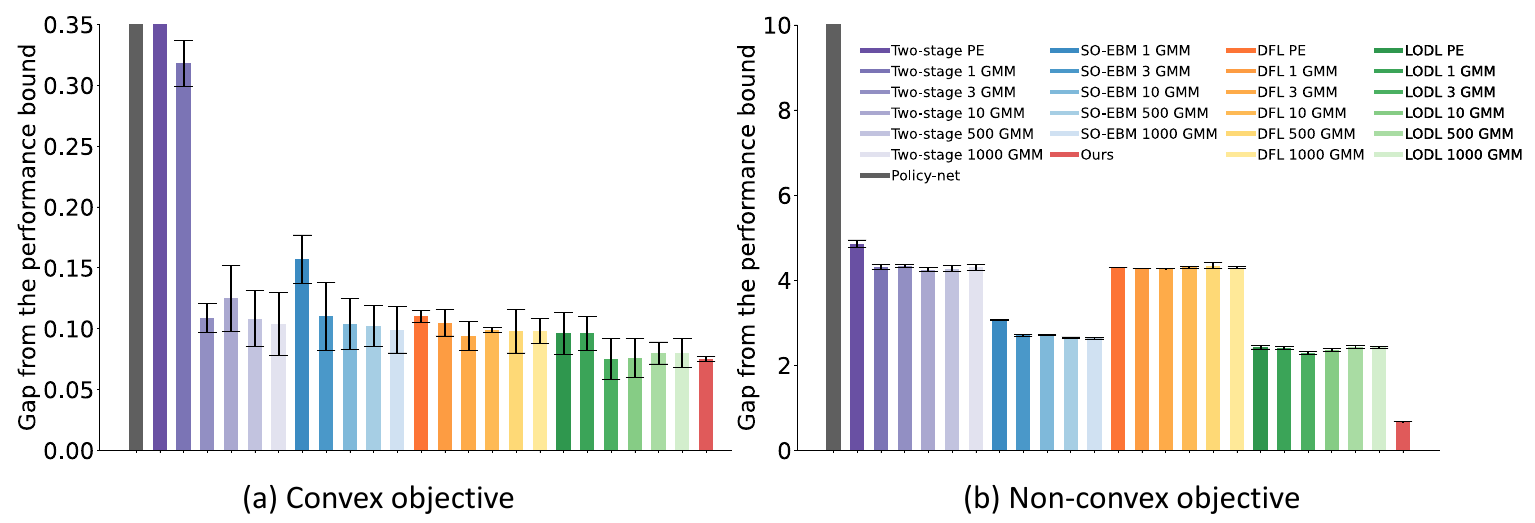}
    \caption{The gap of the model's decision regret from the lower bound of the decision regret on the synthetic data for both the convex and non-convex objectives. `PE' denotes that the forecaster only produces a point estimate for the problem parameter.}
    \label{fig:synthetic}
\end{figure*}

\begin{figure*}[!ht]
    \centering
    \includegraphics[width=0.95\textwidth]{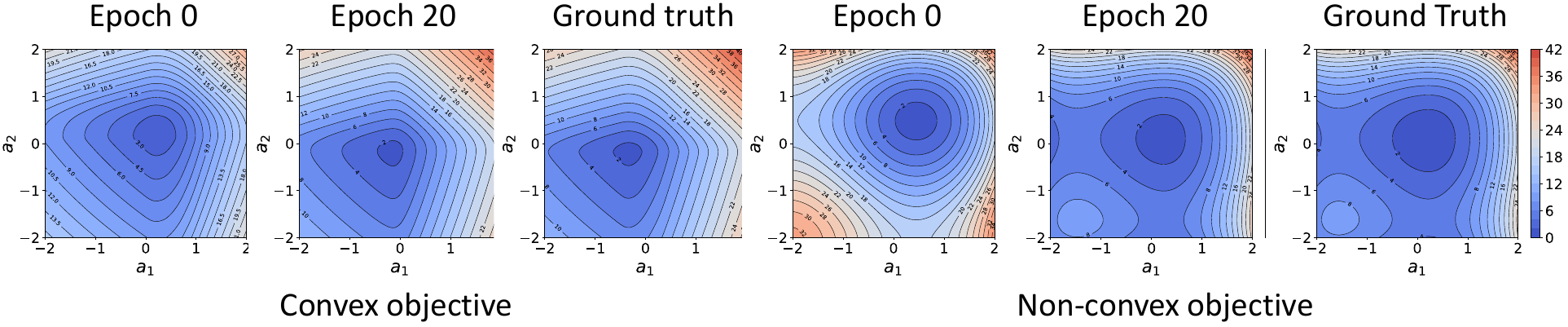}
    \caption{Randomly initialized landscape, \ours recovered landscape and the ground-truth landscape on the synthetic data. The landscape is conditioned on an input feature sampled from the test set.}
    \label{fig:landscape}
\end{figure*}


\section{Experiments}
In this section, we empirically evaluate \ours and conduct experiments in both synthetic and real-world scenarios. Finally, we perform ablation studies to show the effect of each model design in \ours.
\subsection{Synthetic Problems}
To highlight the ability to learn the true expected objective, we first validate our method on a synthetic dataset where the true underlying model is known to us. To simulate the multi-modal scenario in the real world, we generate 5000 feature-parameter pairs using a Gaussian mixture model with three components (3 GMM). We consider both convex and non-convex objectives. The details of the data generation process and the objectives are provided in the Appendix~\ref{s:synthetic}.

\noindent\textbf{Experimental Setup.}

Since we know the true underlying data generation process for this synthetic setting, we compute the lower bound of the decision regret and use the gap of the model's decision regret from this lower bound as the evaluation metric. We compare with the following baselines: (1) A two-stage model trained with negative log-likelihood. (2) DFL~\citep{donti2017task}. 
(3) SO-EBM \citep{kongend}: It uses the energy-based model as a surrogate objective to speed up DFL. (4) Policy-net: It directly maps from the input features to the decision variables by minimizing the task loss using supervised learning. (5) LODL \citep{shah2022learning}: it approximates the decision loss with a surrogate function. 

For the two-stage model, DFL, SO-EBM and LODL, the forecasters use GMM with a different number of components and use 100 samples to estimate the expectation of the objective as we found that more samples bring limited performance gain but lead to longer training time. We also evaluate scenarios where the forecaster provides only a point estimate of the problem parameter, with the exception of SO-EBM, which is originally used in the probabilistic setting. For a fair comparison, we use the same backbone for the encoder of \ours and the forecaster of the baselines and 1000 attention points for both the convex and non-convex objectives. For the two-stage model, DFL and SO-EBM, the forecasters use GMM with a different number of components and use 100 samples to estimate the expectation of the objective as we found that more samples bring little performance gain. For a fair comparison, we use the same backbone for the encoder of \ours and the forecaster of the baselines and 1000 attention points for both the convex and non-convex objectives. Appendix~\ref{s:experiment} provides more details of the experimental setup and model parameters.

\noindent\textbf{Results.}
Fig.~\ref{fig:synthetic} shows the results on both the convex and non-convex objective for all the methods. As we can see, \ours can outperform all the baselines. The improvement of \ours against the baselines becomes more significant on the non-convex objective. Specifically, \ours reduces the gap from the performance bound by $56.5\%$ compared with the strongest baseline LODL. When the baseline methods utilized GMMs with a different number of components, their performance deteriorated, indicating that they suffer from model mismatch errors. In contrast, our method consistently outperformed the baselines, regardless of the number of components they used. This consistency is evidence that our approach can effectively mitigate model mismatch errors. Even when the baseline methods were aligned with the ground-truth model class, our method still outperformed them since we can also avoid the sampling average approximation error at test time.

It's important to note that when the forecaster yields only a point estimate, both existing DFL frameworks and the two-stage method show the worst performance for this imbalanced cost function. This underscores the importance of quantifying uncertainty in the forecaster's predictions, especially in risk-sensitive domains.


Fig.~\ref{fig:landscape}  visualizes the learned expected function and the ground truth expectation on a test sample for both objectives. We found that the \ours can effectively recover the landscape of the ground truth expected cost.

\begin{table}[t]
\centering
\scriptsize
\setlength{\tabcolsep}{0.3em}
\renewcommand{\arraystretch}{0.95}
\resizebox{0.48\textwidth}{!}{%
\begin{tabular}{l|ccc}
\toprule
& \multicolumn{3}{c}{\textbf{Decision Regret}} \\
\cmidrule(lr){2-4}
Method & Power Bidding & Inventory Opt. & Vaccine Dist. \\
\midrule
Policy-net & 489.01 \std{12.39} & 3.96 \std{0.28} & 604 \std{12.30} \\
Two-stage PE & 518.19 \std{14.84} & 3.97 \std{0.15} & 573 \std{10.26} \\
Two-stage 1-GMM & 69.36 \std{4.33} & 3.32 \std{0.10} & 538 \std{9.30} \\
Two-stage 3-GMM & 69.89 \std{1.50} & 3.27 \std{0.08} & 534 \std{8.40} \\
Two-stage 10-GMM & 70.51 \std{2.29} & 3.29 \std{0.05} & 533 \std{7.95} \\
Two-stage 500-GMM & 66.84 \std{1.43} & 3.24 \std{0.07} & 524 \std{7.95} \\
Two-stage 1000-GMM & 65.83 \std{1.70} & 3.27 \std{0.05} & 527 \std{7.65} \\
SO-EBM 1-GMM & 67.32 \std{1.97} & 3.37 \std{0.02} & 512 \std{8.55} \\
SO-EBM 10-GMM & 66.93 \std{2.45} & 3.26 \std{0.03} & 513 \std{7.95} \\
SO-EBM 500-GMM & 67.02 \std{2.16} & 3.37 \std{0.05} & 513 \std{8.70} \\
SO-EBM 1000-GMM & 66.40 \std{2.23} & 3.21 \std{0.07} & 516 \std{9.45} \\
DFL PE & 69.46 \std{1.21} & 3.35 \std{0.03} & 519 \std{7.37} \\
DFL 1-GMM & 66.85 \std{1.47} & 3.36 \std{0.02} & 515 \std{8.25} \\
DFL 3-GMM & 66.60 \std{3.23} & 3.36 \std{0.05} & 513 \std{7.05} \\
DFL 10-GMM & 66.45 \std{2.32} & 3.31 \std{0.01} & 513 \std{7.65} \\
DFL 500-GMM & 65.06 \std{0.88} & 3.24 \std{0.09} & 507 \std{6.60} \\
DFL 1000-GMM & 64.65 \std{3.70} & 3.21 \std{0.07} & 513 \std{7.35} \\
LODL PE & 67.92 \std{1.49} & 3.36 \std{0.06} & 512 \std{7.01} \\
LODL 1-GMM & 66.87 \std{1.36} & 3.34 \std{0.01} & 508 \std{6.23} \\
LODL 3-GMM & 65.75 \std{1.86} & 3.31 \std{0.06} & 506 \std{6.84} \\
LODL 10-GMM & 65.29 \std{1.23} & 3.26 \std{0.02} & 504 \std{6.38} \\
LODL 500-GMM & 64.24 \std{1.45} & 3.22 \std{0.05} & 502 \std{7.02} \\
LODL 1000-GMM & 64.13 \std{2.47} & 3.24 \std{0.04} & 503 \std{7.01} \\
\rowcolor{mygray}
Ours & $\bm{60.90}$ \std{0.60} & $\bm{3.09}$ \std{0.09} & $\bm{492}$ \std{7.05} \\
\bottomrule
\end{tabular}
}
\caption{Decision regret of each method -- \textbf{lower is better}. `PE' denotes point estimate for the parameter.}
\label{table:main}
\end{table}

\subsection{Real-World Problems}
Next, we delve into three real-world problems encompassing both convex and non-convex objectives.

\subsubsection{Experimental Setup.}
\textbf{Wind Power Bidding.}
In this task, a wind power firm engages in both energy and reserve markets,  given the generated wind power $\mathbf{x}\in \mathbb{R}^{24}$ in the last 24 hours. The firm needs to decide the energy quantity $\mathbf{a}_E \in \mathbb{R}^{12}$ to bid and quantity $\mathbf{a}_R \in \mathbb{R}^{12}$ to reserve over the next 12-24 hours in advance, based on the forecasted wind power $\mathbf{y}\in \mathbb{R}^{12}$. The optimization objective is a piecewise function consisting of three segments \citep{manasssakan2022,di2020bidding}, which is to maximize the revenue of the energy sales while minimizing the penalties for decision inaccuracies of overbidding and underbidding.

\textbf{Inventory Optimization.}
In this task, a department store is tasked with predicting the sales $\mathbf{y} \in \mathbb{R}^7$ for the upcoming 7th-14th days based on the past 14 days' sales data $\mathbf{x} \in \mathbb{R}^{14}$ for a specific product, and accordingly, determining the best replenishment strategy $\mathbf{a} \in \mathbb{R}^7$ for each day. The optimization objective is a combination of an under-purchasing penalty, an over-purchasing penalty, and a squared loss between supplies and demands. 


\textbf{Vaccine Distribution for COVID-19.}
During the COVID-19 pandemic, computing a vaccine distribution strategy is one of the most challenging problems for epidemiologists and policymakers. In practice, meta-population Ordinary Differential Equations (ODEs) based epidemiological models~\citep{pei2020differential} are widely used to predict and evaluate the outcomes of different vaccine distribution strategies. These models rely on people mobility data, such as Origin-Destination (OD) matrices, to capture the pandemic spread dynamics across diverse locations~\citep{li2020substantial}.
In this task, given the OD matrices $\mathbf{x}\in \mathbb{R}^{47\times 47 \times 7}$ of last week, \ie~ $\mathbf{x}[i,j,t]$ represents the number of people move from region $i$ to $j$ on day $t$, we need to decide the vaccine distribution $\mathbf{a} \in \mathbb{R}^{47}$ across the 47 regions in Japan with a budget constraint ($\mathbf{a}[i]$ is the number of vaccines distributed to the region $i$). 
 The optimization objective is to 
minimize the total number of infected people over the ODE-drived dynamics, based on the forecasted OD matrices $\mathbf{y}\in \mathbb{R}^{47 \times 47 \times 7}$ for the next week. This task is a challenging non-convex optimization problem due to the nonlinear simulation model. 

Due to space limit, we provide more details of the experimental setup and the optimization objectives in Appendix~\ref{s:experiment}.

\subsubsection{Results.} Table~\ref{table:main} presents the decision regret across three real-world problems, demonstrating that our method consistently outperforms all baselines. Specifically, \ours improves decision regret by \(\{5.0\%, 3.7\%, 2.0\%\}\) compared to the strongest baseline. These three forecasting tasks are characterized by high uncertainty, making it challenging to formulate a precise model assumption. In such scenarios, it is more effective to derive the expected cost function directly from the data, eliminating the need for any parametric distribution assumptions. Moreover, it is clear that simply increasing the number of components in the GMM does not significantly enhance DFL's performance due to increased sample approximation errors. Finally, the probabilistic approach generally exhibits higher and more reliable performance than methods that rely solely on learning a point estimate forecaster.

\subsection{Ablation Study}

\begin{figure*}[!ht]
    \centering
    \includegraphics[width=0.9\textwidth]{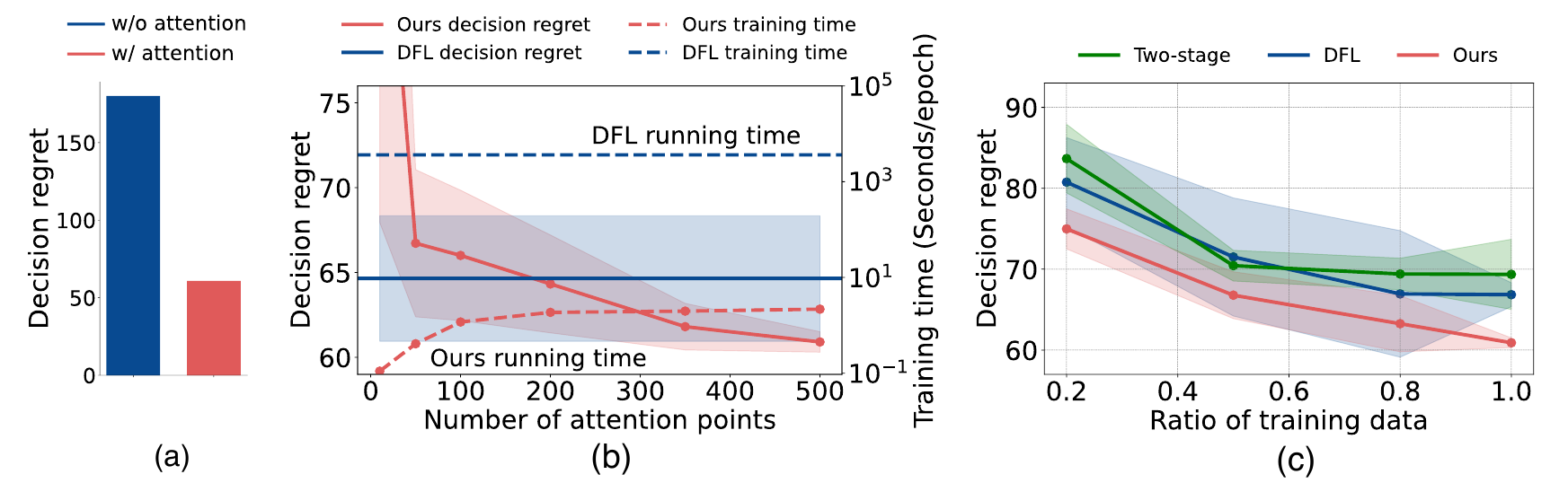}
    \caption{Ablation study on the impact of attention-based architecture, number of attention points, and training data size on the wind power bidding problem.}
    \label{fig:ablation}
\end{figure*}

\begin{figure}[h]
    \centering
    \includegraphics[width=0.5\textwidth]{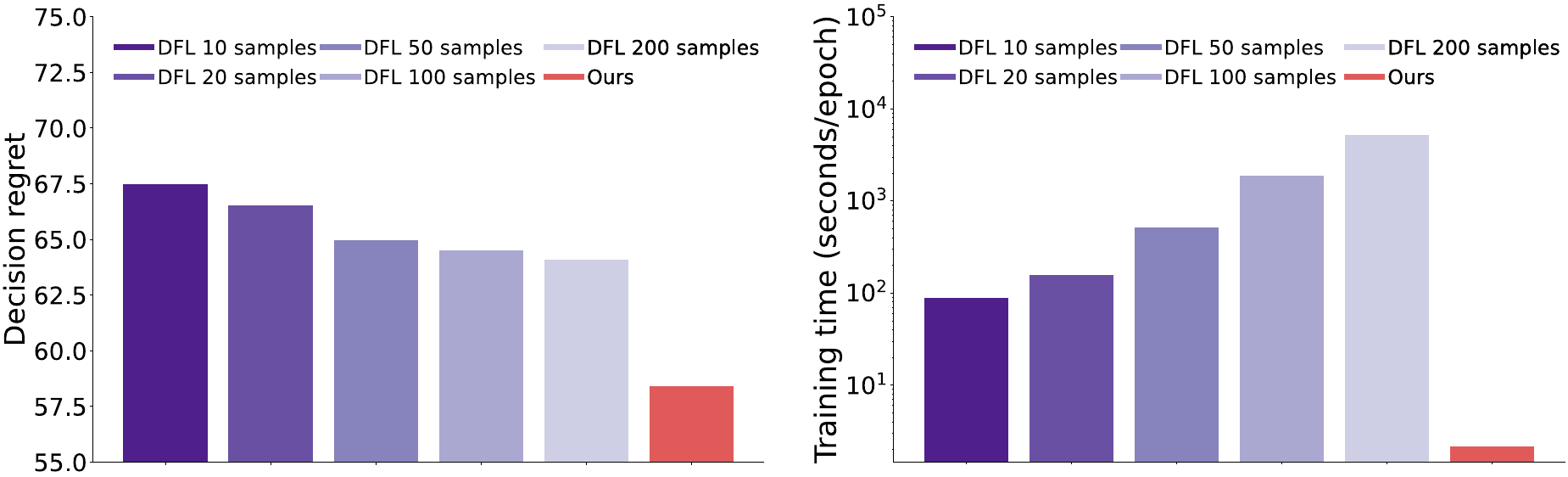}
    \caption{\ours vs. DFL with different numbers of samples: the left figure shows decision regret, while the right figure displays training time.}
    \label{fig:ab1}
\end{figure}

\begin{figure}[h]
    \centering
    \includegraphics[width=0.45\textwidth]{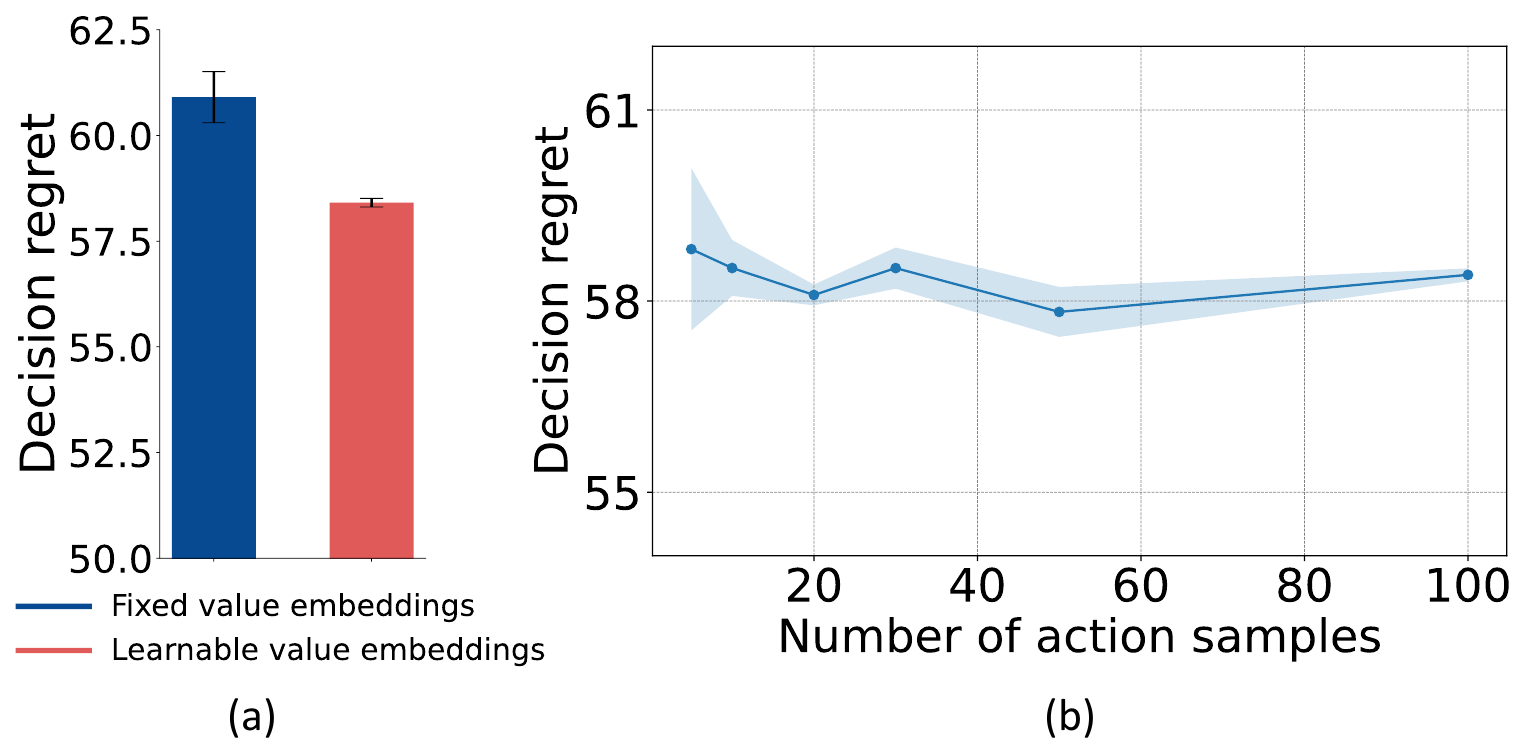}
    \caption{Impact of learnable value embeddings and number of action samples.}
    \label{fig:ab23}
\end{figure}

In this subsection, we investigate each component of \ours via ablation studies on the wind power bidding. 

\emph{Impact of attention-based architecture.} Without the attention-based network architecture, we see a significant performance drop in Fig.~\ref{fig:ablation}(a). This is because, without the attention architecture, the network architecture may not be within the true model class and  thus suffer from high bias error in Proposition~\ref{prop:1}. 

\emph{Impact of number of attention points.} Our model performance can be improved with more attention points as in Fig.\ref{fig:ablation}(b). We also plot the decision regret and training time of DFL. We find that when the number of attention points is over 200, \ours can outperform DFL in terms of the decision regret while being orders of magnitude faster. 

\emph{Impact of training data size.} Our method outperforms baselines constantly with different ratios of training data as shown in Fig.~\ref{fig:ablation}(c). The superior performance is because we use attention-based network architecture to mimic the distribution-based parameterization. Compared with the two-stage model, we are decision-aware; compared with DFL methods, we mitigate the three bottlenecks.

\emph{\ours vs DFL with different number of samples.} The number of samples used to estimate the expected objective in DFL is an important hyperparameter. To investigate its impact, we compare the decision regret and training time of \ours with DFL using different numbers of samples. We use GMM with 1000 components in the DFL forecaster as it achieves the best performance shown in Section 5.2.
As shown in Fig.~\ref{fig:ab1}, when the number of samples for DFL exceeds 100, the performance improvement becomes very marginal (64.52 with 100 samples vs. 64.07 with 200 samples). However, the training time increases significantly (1878 seconds/epoch with 100 samples vs. 5251 seconds/epoch with 200 samples). In contrast, \ours achieves significantly better decision regret (58.41) while being orders of magnitude faster (2.17 seconds/epoch).

\emph{Impact of learnable value embeddings.} In \ours, the value embeddings are initialized with randomly sampled labels from the training set and then updated during the training process. An alternative is to directly use these randomly selected labels and keep the value embeddings fixed during the training process.
We examine whether making the value embeddings learnable improves the performance. The results are shown in Fig.~\ref{fig:ab23}(a). As we can see, with learnable value embeddings, the decision regret of \ours decreases significantly compared with the fixed value embeddings.

 \emph{Impact of number of action samples.} In \ours, we need to sample actions for each $(\mathbf{x},\mathbf{y})$ pair at each training iteration to fit the function.  In this study, we investigate the influence of the number of action samples on the performance. As shown in Fig.\ref{fig:ab23}(b), the decision regret remains stable even for a sample size of 5. Notably, as the number of action samples increases, the variance of the decision regret across different random seeds decreases, indicating improved stability in the results.

\section{Conclusion and Limitations}
We focus on mitigating the three bottlenecks of DFL by differentiating through KKT conditions under the probabilistic setting: (1) model mismatch error, (2) sample average approximation error, and (3) gradient approximation error. To this end, we propose \ours -- the first distribution-free DFL method which does not require any model assumption. \ours adopts a distribution-free training objective that directly learns the expected cost function from the data. To reduce the bias error, we design an attention-based network architecture, drawing inspiration from the distribution-based parameterization of the expected cost function. Empirically, we demonstrate that \ours is effective in a wide range of stochastic optimization problems with either convex or non-convex objectives.

\emph{Limitations.} In our work, we focus on the probabilistic setting where the predictive distribution of the forecasting task has high uncertainty. In this setting, both model mismatch error and sample average approximation error are significant. However, if the forecasting task is relatively straightforward, a simple Gaussian distribution might suffice. For certain objective functions, the expectation under a Gaussian distribution has a closed-form expression. In such cases, existing model-based DFL methods may still be a better choice.

When the number of attention points is large, scalability may become an issue at inference time. This challenge can potentially be alleviated by employing fast attention mechanisms, such as sparse attention (\eg, Longformer; \citep{beltagy2020longformer}) or low-rank approximations (\eg, Linformer; \citep{wang2020linformer}).


\section*{Acknowledgments}

This work was supported in part by the following grants and awards:

\begin{itemize}
  \item NSF IIS-2008334, IIS-2106961, IIS-2403240, CAREER IIS-2028586, CAREER IIS-2144338, RAPID IIS-2027862, Medium IIS-1955883, Medium IIS-2106961, Medium IIS-2403240, Expeditions CCF-1918770, PIPP CCF-2200269, ECCS-2401391
  \item ONR N000142512173
  \item Centers for Disease Control and Prevention Modeling Infectious Diseases In Healthcare program
  \item Dolby faculty research award
  \item Meta faculty gift, and funds/computing resources from Georgia Tech and GTRI.
  \item PPP DA 2025 Flash Funding.
\end{itemize}


\bibliographystyle{plainnat}
\bibliography{refs}
\newpage

\onecolumn

\title{Appendix for \ours: Distribution-Free Decision-Focused Learning}
\maketitle

\appendix

\startcontents[sections]
\printcontents[sections]{l}{1}{\setcounter{tocdepth}{2}}



\section{Training Algorithm}



The full training procedure of \ours is given in Algorithm~\ref{alg:training}.

\begin{algorithm}[H]
\caption{Training Procedure of \ours}
\label{alg:training}
\begin{algorithmic}[1]
\REQUIRE Objective function $f$, feasible set $\mathcal{C}$, training dataset $\mathcal{D}$
\ENSURE Learned encoder, value embeddings, and key embeddings

\STATE Initialize encoder, value embeddings, and key embeddings
\FOR{$t = 1$ to $T$}
    \STATE Sample a mini-batch $B = \{(\mathbf{x}_i, \mathbf{y}_i)\}_{i=1}^{|B|}$ from $\mathcal{D}$
    \FOR{each $(\mathbf{x}_i, \mathbf{y}_i)$ in $B$}
        \STATE Sample actions $\{\mathbf{a}_i^j\}_{j=1}^J$ from the feasible set $C$
        \STATE Compute $g(\mathbf{x}_i, \mathbf{a}_i^j)$ for all $j$, as defined in Eq.~\ref{eq:attention}.
        \STATE Compute the MSE loss for the $i$-th sample:
        \[
        L_i = \frac{1}{J} \sum_{j=1}^J \left( f(\mathbf{x}_i, \mathbf{y}_i) - g(\mathbf{x}_i, \mathbf{a}_i^j) \right)^2
        \]
    \ENDFOR
    \STATE Update encoder, value embeddings, and key embeddings using the aggregated loss $\sum_i L_i$
\ENDFOR
\end{algorithmic}
\end{algorithm}

\section{Constrained Sampling}
\label{s:sampling}

In practice, it's unnecessary to fit the true objective across the entire Euclidean space. Instead, we only need to sample from the constrained space $C$. There are several strategies for this. First, we can employ Markov chain Monte Carlo (MCMC) methods to uniformly sample within $C$, such as Ball Walk~\citep{lovasz1990mixing} and the hit-and-run algorithm ~\citep{belisle1993hit, lovasz1999hit}. Alternatively. we can sample from a relaxed constrained space, such as the encompassing outer box of the original constrained space. This allows us to sample each dimension of $\mathbf{a}$ independently from a uniform distribution. Figure~\ref{fig:sampling} gives an illustration.
\begin{figure}
  \centering
  \begin{center}
    \includegraphics[width=0.25\textwidth]{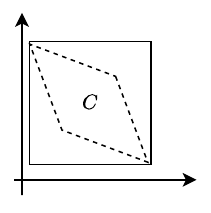}
  \end{center}
  \caption{Relaxed constrained sampling. We can sample from the encompassing outer box of the original constrained space. }
  \label{fig:sampling}
\end{figure}

Consider the following convex constraints:
\begin{align}
\mathbf{A}\mathbf{a} = \mathbf{b}, \quad \mathbf{G}\mathbf{a} \preceq \mathbf{h}.
\end{align}

In particular, instead of directly sampling the full-dimensional decision vector $\mathbf{a}$, we initially output a subset of the variables ${a_1, \cdots, a_d }$, and then deduce the remaining variables by resolving the given set of equations.

To sample from the relaxed constraint space, we initially determine the maximum and minimum values for ${a_1, \cdots, a_d }$, guided by the given inequality constraints. These boundary points can be effortlessly acquired utilizing the Python SciPy package. Following this, we execute uniform sampling between these extremal values for each variable in the set ${a_1, \cdots, a_d }$. Essentially, we transform the polyhedron into a box, simplifying the uniform sampling process.

Furthermore, many predict-then-optimize problems manifest as resource allocation issues wherein the decision variable $\mathbf{a}$ embodies a simplex; in such cases, we can directly sample from the Dirichlet distribution.



\section{Additional Background on Conditional Mean Embedding}
\label{s:Background}

\begin{table}[t]
\centering
\begin{tabular}{c c c}
\toprule[1.5pt]
Variable & $\mathbf{x}$ & $\mathbf{y}$ \\
Domain & $\mathcal{X}$ & $\mathcal{Y}$ \\
Kernel & $\mathcal{R}_{\mathbf{x}}(\mathbf{x},\mathbf{x}')$ & $\mathcal{R}_{\mathbf{y}}(\mathbf{y},\mathbf{y}')$ \\
Feature map & $\phi(\mathbf{x})$/$\mathcal{R}_{\mathbf{x}}(\mathbf{x},\cdot)$ & $\varphi(\mathbf{y})$/$\mathcal{R}_{\mathbf{y}}(\mathbf{y},\cdot)$ \\
Feature matrix & $\Upsilon=(\phi(\mathbf{x}_1), \cdots, \phi(\mathbf{x}_s))$ & $\Phi=(\varphi(\mathbf{y}_1), \cdots, \varphi(\mathbf{y}_s))$ \\
Kernel matrix & $\mathbf{K}=\Upsilon^{\top}\Upsilon$ & $\mathbf{L}=\Phi^{\top}\Phi$ \\
RKHS & $\mathcal{G}$ & $\mathcal{F}$ \\
\bottomrule[1.5pt]
\end{tabular}
\caption{Table of Notations}
\label{table:notation:appendix}
\end{table}

We provide more details about how to compute conditional mean embedding (CME) in this subsection. Table~\ref{table:notation:appendix} presents the notations related to CME.



Let $\mathcal{F}$ be a reproducing kernel Hilbert space (RKHS) over the domain of $\mathbf{y}$ with kernel function $\mathcal{R}_{\mathbf{y}}(\mathbf{y},\mathbf{y}')$ and inner product $\langle \cdot, \cdot \rangle_{\mathcal{F}}$. Its inner product $\langle \cdot, \cdot\rangle_{\mathcal{F}}$ satisfies the reproducing property: $$\langle f(\cdot), \mathcal{R}_{\mathbf{y}}(\mathbf{y}, \cdot) \rangle_{\mathcal{F}}=f(\mathbf{y}),$$
meaning that we can view the evaluation of a function $f\in \mathcal{F}$ at any point $\mathbf{y}$ as an inner product and the linear evaluation operator is given by $\mathcal{R}_{\mathbf{y}}(\mathbf{y},\cdot)$, \ie~the kernel function. Alternatively, $\mathcal{R}_{\mathbf{y}}(\mathbf{y}, \cdot)$ can also be viewed as a feature map $\varphi(\mathbf{y})$ where $\mathcal{R}_{\mathbf{y}}(\mathbf{y},\mathbf{y}')=\langle \varphi(\mathbf{y}), \varphi(\mathbf{y}')\rangle_{\mathcal{F}}$.
Similarly, we can define the RKHS $\mathcal{G}$ over the domain of $\mathbf{x}$ with kernel function $\mathcal{R}_{\mathbf{x}}(\mathbf{x},\mathbf{x}')$.

For a particular $\mathbf{a}$, we denote the corresponding function with respect to $\mathbf{y}$ as $f_\mathbf{a}(\mathbf{y})$. CME projects the conditional distribution to its expected feature map $\mu_{\mathbf{y}|\mathbf{x}}\triangleq \mathbb{E}_{p(\mathbf{y}|\mathbf{x})}[\mathcal{R}_{\mathbf{y}}(\mathbf{y}, \cdot)]$ and evaluates 
the conditional expectation of any RKHS function, $f_\mathbf{a} \in \mathcal{F}$, as an inner product in $\mathcal{F}$ using the reproducing property:
\begin{align}
\mathbb{E}_{p(\mathbf{y}|\mathbf{x})}[f_{\mathbf{a}}] &=\int p(\mathbf{y}|\mathbf{x})\langle \mathcal{R}_{\mathbf{y}}(\mathbf{y},\cdot), f_{\mathbf{a}} \rangle_{\mathcal{F}}d\mathbf{y} \nonumber\\&= \left\langle \int p(\mathbf{y}|\mathbf{x})\mathcal{R}_{\mathbf{y}}(\mathbf{y}, \cdot)\mathrm{d}\mathbf{y}, f_{\mathbf{a}} \right\rangle_{\mathcal{F}} \nonumber\\&=\langle \mu_{\mathbf{y}|\mathbf{x}}, f_{\mathbf{a}}\rangle_{\mathcal{F}}.
\end{align}

Assume that for all $f_{\mathbf{a}}\in \mathcal{F}$, the conditional expectation $\mathbb{E}_{p(\mathbf{y}|\mathbf{x})}[f_{\mathbf{a}}(\mathbf{y})]$ is an element of $\mathcal{G}$, the conditional embedding can be estimated with a finite dataset $\{\mathbf{x}_s, \mathbf{y}_s\}_{s=1}^S$ as \citet{song2013kernel, song2009hilbert}:
\begin{align}
    \hat{\mu}_{\mathbf{y}|\mathbf{x}} =  \Phi(\mathbf{K}+\lambda \mathbf{I})^{-1}\Upsilon^\top\phi(\mathbf{x}) =  \sum_{s=1}^S \beta_s(\mathbf{x}) \mathcal{R}_{\mathbf{y}}(\mathbf{y}_s,\cdot),
    \label{eq:cme_kernel}
\end{align}
where $\Phi=(\mathcal{R}_{\mathbf{y}}(\mathbf{y}_1,\cdot), \cdots, \mathcal{R}_{\mathbf{y}}(\mathbf{y}_S,\cdot))$ is the feature matrix; $\mathbf{K}=\Upsilon^\top \Upsilon$ is the Gram matrix for samples from variable $\mathbf{x}$ with $\Upsilon=(\mathcal{R}_{\mathbf{x}}(\mathbf{x}_1,\cdot), \cdots, \mathcal{R}_{\mathbf{x}}(\mathbf{x}_S,\cdot))$; 
$\lambda$ is the additional regularization parameter to avoid overfitting. Though the assumption $\mathbb{E}_{p(\mathbf{y}|\mathbf{x})}[f_{\mathbf{a}}(\mathbf{y})]\in \mathcal{G}$ is not necessarily true for continuous domains, existing works treat the expression as an approximation \citep{song2009hilbert} and works well in practice.

 
 One advantage of CME is that $\hat{\mu}_{\mathbf{y}|\mathbf{x}}$ can converge to $\mu_{\mathbf{y}|\mathbf{x}}$ in the RKHS norm at an overall rate of $\mathcal{O}(S^{-\frac{1}{2}})$ \citep{song2009hilbert}, which is independent of the input dimensions. This property let CME works well in the high-dimensional space. 

As we can see from Eq.~\ref{eq:cme_kernel}, the empirical estimator of CME, ${\hat \mu}_{\mathbf{y}|\mathbf{x}}$,  applies non-uniform
weights, $\beta_s$, on observations which are, in turn, determined by the conditioning variable $\mathbf{x}$.

\section{Proof of Proposition 1}
\label{s:Prop1}

\newtheorem{prop}{Proposition}

\begin{prop}\label{prop:1}
The expected MSE of the optimal solution $g^*$ on the test set is: 
\begin{align}
 \text{MSE}_{\rm test}  =  \underbrace{\mathbb{E}_{\mathcal{D}'}\left [ \left(g^*_{\mathcal{D}'}(\mathbf{x,\mathbf{a}})- \mathbb{E}_{p(\mathbf{y}|\mathbf{x})}[f(\mathbf{y},\mathbf{a})]   \right)^2 \right]}_{\text{Bias}}\nonumber\\ 
 + \underbrace{\mathbb{E}_{\mathcal{D}'}\left[ \left(g^*_{\mathcal{D}'}(\mathbf{x,\mathbf{a}})-  \mathbb{E}_{\mathcal{D}'}[g^*_{\mathcal{D}'}(\mathbf{x,\mathbf{a}})]  \right)^2 \right]}_{\text{Variance}}, \nonumber
\end{align} 
where $\mathcal{D}'$ denotes the training dataset  augmented with the sampled actions $\mathbf{a}$. 
\end{prop}

\begin{proof}
The training set consists of the given $\mathcal{D}=\{\mathbf{x}_i,\mathbf{y}_i\}_{i=1}^N$ augmented with the sampled actions $\mathbf{a}$. We denote the augmented dataset as $\mathcal{D}'$. We assume the fitted function is in a hypothesis $g^*(\mathbf{x},\mathbf{a})$. Let $g^*_{\mathcal{D}'}(\mathbf{x},\mathbf{a})$ denote the function fitted on the dataset $\mathcal{D}'$.
The expectation of the mean squared error (MSE) for a given unseen test sample, over all possible learning sets, is:
\begin{align}
& \hspace{1.1em} \mathbb{E}_{\mathcal{D}'}[(\mathbb{E}_{p(\mathbf{y}|\mathbf{x})}[f(\mathbf{y},\mathbf{a})]-g^*_{\mathcal{D}'}(\mathbf{x},\mathbf{a}))^2]   \nonumber \\
& = \mathbb{E}_{\mathcal{D}'}[(\underbrace{\mathbb{E}_{p(\mathbf{y}|\mathbf{x}^*)}[f(\mathbf{y},\mathbf{a})] - \mathbb{E}_{\mathcal{D}'}[g^*_{\mathcal{D}'}(\mathbf{x,\mathbf{a}})]}_{a} \nonumber \\ &+ \underbrace{\mathbb{E}_{\mathcal{D}'}[g^*_{\mathcal{D}'}(\mathbf{x,\mathbf{a}})]-g^*_{\mathcal{D}'}(\mathbf{x,\mathbf{a}})}_{b} )^2 ] 
\nonumber \\ &= \mathbb{E}_{\mathcal{D}'}[(a+b)^2] \nonumber \\
&= \mathbb{E}_{\mathcal{D}'}[a^2]+\mathbb{E}_{\mathcal{D}'}[b^2] + \mathbb{E}_{\mathcal{D}'}[2ab] \nonumber
\end{align}

The first two terms represent the bias and variance errors respectively:
$$\mathbb{E}_{\mathcal{D}'}[a^2] = \mathbb{E}_{\mathcal{D}'}\left[ \left(g^*_{\mathcal{D}'}(\mathbf{x,\mathbf{a}})- \mathbb{E}_{p(\mathbf{y}|\mathbf{x})}[f(\mathbf{y},\mathbf{a})]   \right)^2 \right] =\text{Bias}^2(g^*).$$
$$\mathbb{E}_{\mathcal{D}'}[b^2] = \mathbb{E}_{\mathcal{D}'}\left[ \left(g^*_{\mathcal{D}'}(\mathbf{x,\mathbf{a}})-  \mathbb{E}_{\mathcal{D}'}[g^*_{\mathcal{D}'}(\mathbf{x,\mathbf{a}})]  \right)^2 \right] =\text{Variance}(g^*),$$


Next, we prove the cross-term $\mathbb{E}_{\mathcal{D}'}[2ab]=0$. To simplify the notation, let ${ \overline g}$ denote $\mathbb{E}_{\mathcal{D}'}[g^*_{\mathcal{D}'}(\mathbf{x,\mathbf{a}})]$; $g$  denote $g^*_{\mathcal{D}'}(\mathbf{x,\mathbf{a}})$; ${\tilde f}$ denote $\mathbb{E}_{p(\mathbf{y}|\mathbf{x})}[f(\mathbf{y},\mathbf{a})]$. Then we can obtain:
\begin{align}
&\mathbb{E}_{\mathcal{D}'}\left[2\left(g - {\overline g}\right)\left( {\overline g} - {\tilde f} \right)\right] \nonumber\\
&= 2\cdot\mathbb{E}_{\mathcal{D}'}[g \cdot {\overline g} - g \cdot {\tilde f} - {\overline g} \cdot {\overline g} + {\overline g} \cdot {\tilde f}] \nonumber \\
& = 2\cdot\mathbb{E}_{\mathcal{D}'}[g]\cdot {\overline g}- 2\cdot\mathbb{E}_{\mathcal{D}'}[g]\cdot{\tilde f} -2\cdot\mathbb{E}_{\mathcal{D}'}[{\overline g}^2] + 2\cdot{\tilde f}\cdot\mathbb{E}_{\mathcal{D}'}[{\overline g}] \nonumber \\
& = 2 \cdot {\overline g}^2 - 2\cdot {\overline g}\cdot{\tilde f} - 2\cdot{\overline g}^2 + 2\cdot{\tilde f}\cdot{\overline g} \nonumber \\
& = 0 \nonumber
\end{align}

Hence, the expectation of the MSE for a given test sample $\mathbf{x}^*$ is expressed as:
\begin{align}
 \text{MSE}_{\rm test} & = \mathbb{E}_{\mathcal{D}'}[(\mathbb{E}_{p(\mathbf{y}|\mathbf{x})}[f(\mathbf{y},\mathbf{a})]-g^*_{\mathcal{D}'}(\mathbf{x,\mathbf{a}}))^2] \nonumber \\
  & =  \underbrace{\mathbb{E}_{\mathcal{D}'}\left [ \left(g^*_{\mathcal{D}'}(\mathbf{x,\mathbf{a}})- \mathbb{E}_{p(\mathbf{y}|\mathbf{x})}[f(\mathbf{y},\mathbf{a})]   \right)^2 \right]}_{\text{Bias}} \nonumber\\ &+ \underbrace{\mathbb{E}_{\mathcal{D}'}\left[ \left(g^*_{\mathcal{D}'}(\mathbf{x,\mathbf{a}})-  \mathbb{E}_{\mathcal{D}'}[g^*_{\mathcal{D}'}(\mathbf{x,\mathbf{a}})]  \right)^2 \right]}_{\text{Variance}} 
  \label{eq:mse}
\end{align}

Since the training dataset consists of $\mathbf{x}, \mathbf{y}, \mathbf{a}$, and each $\mathbf{y}$ corresponds to a  specific $\mathbf{x}$ from $\mathcal{D}$, we can replace the expectation $\mathbb{E}_{\mathcal{D}'}[\cdot]$ in Eq.~\ref{eq:mse} with $\mathbb{E}_{\mathbf{x},\mathbf{a}}[\cdot]$ and recover Proposition~\ref{prop:1}.

\end{proof}

\section{Proof of Proposition 2}
\label{s:Prop2}

\begin{prop}\label{prop:my_proposition}
It holds for any $\mathbf{x}$ and $\mathbf{a}$, the function $g(\mathbf{x},\mathbf{a})$ defined by the softmax attention in Eq.~\ref{eq:attention} $\mathbb{E}_{\hat{p}_{\mathcal{R}}(\mathbf{y}|\mathbf{x})}[f(\mathbf{y},\mathbf{a})]=g(\mathbf{x},\mathbf{a})$. Here,
$\hat{p}_{\mathcal{R}}(\mathbf{y}|\mathbf{x})$ is a  parameterization restriction of $p(\mathbf{y}|\mathbf{x})$. 
\end{prop}

\begin{proof}
In order to ensure that $\hat{p}_{\mathcal{R}}(\mathbf{y}|\mathbf{x})$ is a valid parameterization of $p(\mathbf{y}|\mathbf{x})$, we define it as a  conditional kernel density estimator (KDE) as follows,
\begin{align}
    {\hat p}_{\mathcal{R}}(\mathbf{y}|\mathbf{x}) = \frac{ \sum_{s=1}^S\mathcal{R}_{\mathbf{x}}(\mathbf{k}_s, \mathbf{q}(\mathbf{x}))\mathcal{R}_{\mathbf{y}}(\mathbf{y}_s, \mathbf{y})}{\sum_{s=1}^S \mathcal{R}_{\mathbf{x}} (\mathbf{k}_s, \mathbf{q(\mathbf{x})})},
\end{align}

Then, we  can obtain 
\begin{align}
    \mathbb{E}_{{\hat p}_{\mathcal{R}}(\mathbf{y}|\mathbf{x})}[f(\mathbf{y}, \mathbf{a})] &= \int \frac{ \sum_{s=1}^S\mathcal{R}_{\mathbf{x}}(\mathbf{k}_s, \mathbf{q}(\mathbf{x}))\mathcal{R}_{\mathbf{y}}(\mathbf{y}_s, \mathbf{y})}{\sum_{s=1}^S \mathcal{R}_{\mathbf{x}} (\mathbf{k}_s, \mathbf{q(\mathbf{x})})} f(\mathbf{y},\mathbf{a}) \mathrm{d} \mathbf{y} \nonumber \\
    & = \frac{ \sum_{s=1}^S\mathcal{R}_{\mathbf{x}}(\mathbf{k}_s, \mathbf{q})\int \mathcal{R}_{\mathbf{y}}(\mathbf{y}_s, \mathbf{y}) f(\mathbf{y},\mathbf{a}) \mathrm{d} \mathbf{y}}{\sum_{s=1}^S \mathcal{R}_{\mathbf{x}} (\mathbf{k}_s, \mathbf{q(\mathbf{x})})} \nonumber \\
    & = \frac{ \sum_{s=1}^S\mathcal{R}_{\mathbf{x}}(\mathbf{k}_s, \mathbf{q})\int  \mathcal{R}_{\mathbf{z}}(f(\mathbf{y}_s, \mathbf{a}), \mathbf{z})\mathbf{z}\mathrm{d}\mathbf{z}}{\sum_{s=1}^S \mathcal{R}_{\mathbf{x}} (\mathbf{k}_s, \mathbf{q(\mathbf{x})})} \nonumber \\
   & = \frac{ \sum_{s=1}^S\mathcal{R}_{\mathbf{x}}(\mathbf{k}_s, \mathbf{q})f(\mathbf{y}_s, \mathbf{a})}{\sum_{s=1}^S \mathcal{R}_{\mathbf{x}} (\mathbf{k}_s, \mathbf{q(\mathbf{x})})}
\end{align}

The second last equation comes from the result of the change of variable by setting $\mathbf{z}=f(\mathbf{y},\mathbf{a})$. The last equation comes from the assumption that $\mathcal{R}_{\mathbf{z}}(\mathbf{z}_s, \mathbf{z})$ is symmetric.

When $\mathcal{R}_{\mathbf{x}}({\mathbf{k}, \mathbf{q}})$ is an exponential kernel. \ie~ $\mathcal{R}_{\mathbf{x}}({\mathbf{k}, \mathbf{q}})=\text{exp}(\frac{\mathbf{q}^\top\mathbf{k}}{\sqrt{d}})$, we can obtain
\begin{align}
\mathbb{E}_{{\hat p}_{\mathcal{R}}(\mathbf{y}|\mathbf{x})}[f(\mathbf{y}, \mathbf{a})] &= \frac{ \sum_{s=1}^S\mathcal{R}_{\mathbf{x}}(\mathbf{k}_s, \mathbf{q}(\mathbf{x}))f(\mathbf{y}_s, \mathbf{a})}{\sum_{s=1}^S \mathcal{R}_{\mathbf{x}} (\mathbf{k}_s, \mathbf{q}(\mathbf{x}))} \nonumber\\
& = \frac{ \sum_{s=1}^S \exp\left(\frac{\mathbf{q}(\mathbf{x})^\top\mathbf{k}_s}{\sqrt{d}}\right) f(\mathbf{y}_s, \mathbf{a})}{\sum_{s=1}^S \exp\left(\frac{\mathbf{q}(\mathbf{x})^\top\mathbf{k}_s}{\sqrt{d}}\right)} \nonumber \\
&= \text{Softmax}\left(\left[\frac{\mathbf{q}(\mathbf{x})^\top\mathbf{k}_1}{\sqrt{d}}, \cdots, \frac{\mathbf{q}(\mathbf{x})^\top\mathbf{k}_S}{\sqrt{d}}\right]\right)^\top \nonumber \\ &[f(\mathbf{v}_1, \mathbf{a}), \cdots, f(\mathbf{v}_S, \mathbf{a})] \nonumber \\
& = g(\mathbf{x},\mathbf{a}).
\end{align}
The second last equation comes from the definition of the softmax function and replacing the notation $\mathbf{y}$ with $\mathbf{v}$ which is commonly used in the existing literature.

\end{proof}

\section{Experimental Details}
\label{s:experiment}
\subsection{Computing Infrastructure}
\label{s:computing}
System: Ubuntu 18.04.6 LTS; Python 3.9; Pytorch
1.11. CPU: Intel(R) Xeon(R) Silver 4214 CPU @ 2.20GHz. GPU: GeForce GTX 2080 Ti.

\subsection{Synthetic Data}
\label{s:synthetic}
\noindent\textbf{Data generation process:}
We generate the synthetic dataset following a mixture of three Gaussians:
\begin{align}
    &\mathbf{x} \sim \mathcal{U}^2[-1,1], \quad \nonumber \\
    &\mathbf{y} \sim 0.3\mathcal{N}(\mathbf{A}_1\mathbf{x}, 0.1\cdot\mathbf{I}) + 0.3 \mathcal{N}(\mathbf{A}_2\mathbf{x}, 0.1\cdot\mathbf{I}) + 0.4 \mathcal{N}(\mathbf{A}_3\mathbf{x}, 0.1\cdot\mathbf{I}),
\end{align}

where the elements of $\mathbf{A}_1, \mathbf{A}_2, \mathbf{A}_3 \in \mathbb{R}^{2\times 2}$ are uniformly sampled from $\mathcal{U}[0,1].$

We generate 5000 $(\mathbf{x},\mathbf{y})$ pairs, randomly dividing them into a training set (70\%, 3500 pairs), and equal validation and testing sets (15\% each, 750 pairs).
 
\noindent\textbf{Optimization objective:} We consider both the convex and non-convex objectives.

Convex objective:
\begin{align}
   &\text{minimize}_{\mathbf{a}\in \mathbb{R}^2} \mathbb{E}_{p(\mathbf{y}|\mathbf{x})} \left[ 
   \sum_{i=1}^2 \left( 5(\mathbf{y}[i]-\mathbf{a}[i])_{+} + 20(\mathbf{a}[i]-\mathbf{y}[i])_{+} 
   + 0.5(\mathbf{y}[i]-\mathbf{a}[i])_{+}^2 + 0.2(\mathbf{a}[i]-\mathbf{y}[i])_{+}^2 \right)
   \right] \nonumber \\
   &\text{subject to}  \quad -1 \le \mathbf{a}[i] \le 1, \forall i. \nonumber
\end{align}

Non-convex objective:
\begin{align}
    &\text{minimize}_{\mathbf{a} \in \mathbb{R}^2} \mathbb{E}_{p(\mathbf{y}|\mathbf{x})}\sum_{i=1}^2\left[10(\mathbf{y}[i]-\mathbf{a}[i])_{+}^2 + 2(\mathbf{a}[i]-\mathbf{y}[i])_{+}^2 + 4\mathbf{a}[i]^3\right] \nonumber\\
    &\text{subject to} \quad -2 \le \mathbf{a}[i] \le 2, \forall i, \nonumber
\end{align}
where $(v)_{+}$ denote $\text{max}\{v,0\}$.

\noindent\textbf{Solver at test time:} At test time, for a fair comparison, we use the same optimization solver for all the methods. Specifically, we use projected gradient descent and the gradient update step adopts the Adam \citep{kingma2015adam} optimizer. The learning rate is $0.01$ and we repeat $500$ iterations. We empirically found that this solver solves this optimization problem very well.

\noindent\textbf{Model Hyperparameters:}
For the two-stage model, DFL, LODL and SO-EBM, the forecaster uses GMM with a different number of components and use 100 samples to estimate the expectation of the objective as we found that more samples bring little performance gain. The forecaster uses a neural network with one hidden layer as the feature extractor which is further stacked by a linear layer. This network has a hidden size of 128, employing ReLU as the nonlinear activation function. The forecaster outputs the mean, log variance, and weight for each GMM component. During training, we sample from the GMM using the Gumbel softmax trick \citep{jang2017categorical} to make the sampling process differentiable. 
SO-EBM draws 512 samples from the proposal distribution to estimate the gradient of the model parameters. The proposal distribution is a mixture of Gaussians with 3 components where the variances are $\{0.01, 0.02, 0.05\}$.

For a fair comparison, \ours uses the same feature extractor for the encoder. The attention architecture uses 1000 attention points for both the convex and non-convex objectives. During training, \ours samples 100 actions $\mathbf{a}$ uniformly from the constrained space, \ie~ the box, for each $(\mathbf{x},\mathbf{y})$ pair at each iteration for function fitting.

\noindent\textbf{Model Optimization:}
We use the Adam~\citep{kingma2015adam} algorithm for model optimization. The number of training epochs is 50. The learning rate for all the methods is $10^{-3}$. DFL, LODL and SO-EBM use the two-stage model as the pre-trained model for faster training convergence.

\subsection{Wind Power Bidding}
\label{s:wind}
\noindent\textbf{Optimization objective:} 
In this task, a wind power firm engages in both energy and reserve markets,  given the generated wind power $\mathbf{x}\in \mathbb{R}^{24}$ in the last 24 hours. The firm needs to decide the energy quantity $\mathbf{a}_E \in \mathbb{R}^{12}$ to bid and quantity $\mathbf{a}_R \in \mathbb{R}^{12}$ to reserve over the upcoming 12th-24th hours in advance, based on the forecasted wind power $\mathbf{y}\in \mathbb{R}^{12}$.
The optimization objective is to maximize the profit which is a piecewise function consisting of three segments~\citep{di2020bidding1,manasssakan20221, zhuang2023dygen}:
{\small
\begin{align}
&\text{maximize}_{\mathbf{a}_E \in \mathbb{R}^{12}, \mathbf{a}_R \in \mathbb{R}^{12}} \mathbb{E}_{p(\mathbf{y}|\mathbf{x})}\sum_{i=1}^{12} P\mathbf{y}[i] - \nu \mathbf{a}_{R}[i] \nonumber\\ &+ 
\begin{cases} 
&-\Delta P_{\rm up,1}(\mathbf{a}_{E}[i]-\mathbf{a}_{R}[i]-\mathbf{y}[i]) -\Delta P_{\rm up,2}(\mathbf{a}_{E}[i]-\mathbf{a}_{R}[i]-\mathbf{y}[i])^2 \\ &- \mu \mathbf{a}_{R}[i] -F, \text{if } \mathbf{y}[i] < \mathbf{a}_{E}[i] - \mathbf{a}_{R}[i] \nonumber\\
&- \mu (\mathbf{a}_{E}[i]-\mathbf{y}[i]), \text{if}\ \mathbf{a}_{E}[i]-\mathbf{a}_{R}[i] \leq \mathbf{y}[i] \leq \mathbf{a}_{E}[i] \\
&-\Delta P_{\rm down}(\mathbf{y}[i]-\mathbf{a}_{E}[i]),\text{if } \mathbf{y}[i] > \mathbf{a}_{E}[i] 
\end{cases}
\end{align}
}
\hspace*{1em}$\text{subject to}  \quad E_{\rm min} \le \mathbf{a}_E[i] \le E_{\rm max},  R_{\rm min} \le \mathbf{a}_R[i] \le R_{\rm max}, \quad \forall i.$

$P$ is the regular price of the wind energy sold, $\mathbf{y}[i]$ is the energy generated during period $i$,  $\mathbf{a}_{E}[i]$ and $\mathbf{a}_{R}[i]$ are the bid and up reserve energy volumes for period $i$, respectively. $\nu$ corresponds to the opportunity cost when the company participates in the reserve markets, and $\mu$ is the deploy price of the reserved energy. This structure encapsulates three market participation scenarios. In the scenario where $\mathbf{y}[i] < \mathbf{a}_{E}[i] - \mathbf{a}_{R}[i]$, the company overbids, consequently deploying all reserved energy and facing a linear overbidding penalty, a quadratic overbidding penalty and a constant penalty determined by coefficients $\Delta P_{\rm up,1}$, $\Delta P_{\rm up,2}$, and $F$. If $ \mathbf{a}_{E}[i]-\mathbf{a}_{R}[i] \leq \mathbf{y}[i] \leq \mathbf{a}_{E}[i]$, the company meets its bid by deploying reserve market energy, thereby avoiding penalties. In this case, the company only needs to pay the deployment fee for the reserved energy. However, when $\mathbf{y}[i] > \mathbf{a}_{E}[i]$, the company underbids, resulting in the selling of surplus electricity at a discount and incurring losses defined by the coefficient $\Delta P_{\rm down}$. 
We set $P$ as 100, according to the average bidding price obtained from Nord Pool, a European power exchange. $\nu$ and $\mu$ are 20 and 110 respectively, as a general setting~\citep{di2020bidding1,manasssakan20221}. The value of $\Delta P_{\rm up,1}$, $\Delta P_{\rm up,2}$, $\Delta P_{\rm down}$ and $F$ are set to 200, 100, 20 and 10, to ensure an effective penalty. $E_{\rm min}=0$, $R_{\rm min}= 0.15$, and $E_{\rm max}=R_{\rm max}=4$.

According to the optimality condition, the optimal $\mathbf{a}_{R}[i]$ is always equal to $R_{\rm min}$ for all $i$. Therefore, we only need to determine the decision variable $\mathbf{a}_E$.

We use the wind power generation dataset of the German energy company TenneT during 08/23/2019 to 09/22/2020 \footnote{The dataset is available at: \url{https://www.kaggle.com/datasets/jorgesandoval/wind-power-generation?select=TransnetBW.csv}}. 
The split ratio of the training dataset, validation dataset, and test datset are 64\%, 16\%, 20\%, respectively.

\noindent\textbf{Solver at test time:} At test time, for a fair comparison, we use the same optimization solver for all the methods. Specifically, we use projected gradient descent and the gradient update step adopts the Adam~\citep{kingma2015adam} optimizer. The learning rate is $0.1$ and we repeat $500$ iterations. We empirically found that this solver solves this optimization problem very well.

\noindent\textbf{Model Hyperparameters:}
For the two-stage model, DFL and SO-EBM, the forecaster uses GMM with a different number of components and use 100 samples to estimate the expectation of the objective as we found that more samples bring little performance gain. The forecaster uses a two-layer long short-term memory network (LSTM)  as the feature extractor which is further stacked by a linear layer. The network has a hidden size of 256. It takes the historical wind power in the last 24 hours as input features and outputs the forecasted wind power for the 12th to 24th hours in the future. The forecaster outputs the mean, log variance, and weight for each GMM component. During training, we sample from the GMM using the Gumbel softmax trick~\citep{jang2016categorical} to make the sampling process differentiable. 
SO-EBM draws 512 samples from the proposal distribution to estimate the gradient of the model parameters. The proposal distribution is a mixture of Gaussians with 3 components where the variances are $\{0.02, 0.05, 0.1\}$.

For a fair comparison, \ours uses the same LSTM architecture as the encoder and 500 attention points. During training, we sample 100 actions $\mathbf{a}$ uniformly  from the constrained space for each $(\mathbf{x},\mathbf{y})$ pair at each iteration.


\noindent\textbf{Model Optimization:}
We use the Adam~\citep{kingma2015adam} algorithm for model optimization. The number of training epochs is 200. The learning rate for all the methods is $10^{-3}$. DFL and SO-EBM use the two-stage model as the pre-trained model for faster training convergence.

\subsection{COVID-19 Vaccine Distribution}
\label{s:covid}
\noindent\textbf{Optimization objective:} In this task, given the OD matrices $\mathbf{x}\in \mathbb{R}^{47\times 47 \times 7}$ of last week, \ie~ $\mathbf{x}[i,j,t]$ represents the number of people move from region $i$ to $j$ on day $t$, we need to decide the vaccine distribution $\mathbf{a} \in \mathbb{R}^{47}$ across the 47 regions in Japan with a budget constraint ($\mathbf{a}[i]$ is the number of vaccines distributed to the region $i$). 
 The optimization objective is to 
minimize the total number of infected people over the ODE-drived dynamics, based on the forecasted OD matrices $\mathbf{y}\in \mathbb{R}^{47 \times 47 \times 7}$ for the next week. 

We want to distribute the vaccine over each county to minimize the number of infected cases. The number of infected cases is given by a metapopulation SEIRV model \citep{li2020substantial1,pei2020differential1}, denoted by $\text{Simulator}(\cdot,\cdot)$:
\begin{align}
    \argmin_{\mathbf{a}\in \mathbb{R}^{47}} \mathbb{E}_{p(\mathbf{y}|\mathbf{x})}[\text{Simulator}(\mathbf{y}, \mathbf{a})], \nonumber\\
    \text{Subject to} \quad \sum_i \mathbf{a}[i] \le \text{Budget}, \mathbf{a}[i] \ge 0. \nonumber
\end{align}
We use the OD matrices dataset of Japan \footnote{The dataset is available at \url{https://github.com/deepkashiwa20/ODCRN/tree/main/data}} during 04/01/2020 to 02/28/2021. The split ratio of the training dataset, validation dataset, and test datset are 64\%, 16\%, 20\%, respectively. We set the budget as $5\times10^{6}$.

\textbf{Details of the simulator:}
The SEIRV model is an epidemiological model used to predict and understand the spread of infectious diseases. It divides the population into five compartments: Susceptible (S), Exposed (E), Infectious (I), Recovered (R) and Vaccined (V).
The model is defined by a set of differential equations that describe the transitions between these compartments. There are four hyperparameters in the SEIRV model: 
\begin{itemize}
    \item  $\beta$ - Transmission rate: Represents the average number of contacts per person per unit of time multiplied by the probability of disease transmission in a contact between a susceptible and an infectious individual.
\item $\sigma$ - Latent rate (or the inverse of the incubation period): The rate at which exposed individuals progress to the infectious state. The incubation period is the time it takes for an individual to become infectious after exposure.

 \item $\gamma$ - Recovery rate (or the inverse of the infectious period): The rate at which infectious individuals recover or die and transition to the recovered state. The infectious period is the time during which an infected individual can transmit the disease.

\item $N$ - Total population: The sum of individuals in all compartments (S, E, I, R, V).
\end{itemize}

When considering mobility flow among different regions, we need to adapt the SEIRV model to account for the movement of individuals between regions. In this case, the model becomes a spatially explicit, multi-region SEIRV model. Each region will have its own SEIRV model, and the flow of individuals between regions will affect the dynamics of the compartments. Specifically, for each region $k=1,\cdots, K$, we have:
\begin{align}
\frac{\mathrm{d}\mathbf{S}[k]}{\mathrm{d}t} &= -\bm{\beta}[k] \frac{\mathbf{S}[k] \cdot\mathbf{I}[k]}{\mathbf{N}[k]} - \frac{\mathbf{S}[k]}{\mathbf{S}[k]+\mathbf{E}[k]}\cdot\frac{\mathbf{a}[k]}{T} \nonumber\\ &+ \sum_{i \neq k} \tilde{\mathbf{y}}[i,k, t] \cdot\mathbf{S}[i] - \sum_{j \neq k} \tilde{\mathbf{y}}[k,j,t] \cdot\mathbf{S}[k], \nonumber\\
\frac{\mathrm{d}\mathbf{E}[k]}{\mathrm{d}t} &= \bm{\beta}[k] \frac{\mathbf{S}[k] \cdot\mathbf{I}[k]}{\mathbf{N}[k]} - \bm{\sigma}[k] \cdot\mathbf{E}[k] - \frac{\mathbf{E}[k]}{\mathbf{S}[k]+\mathbf{E}[k]} \cdot\frac{\mathbf{a}[k]}{T} \nonumber\\&+ \sum_{i \neq k} \tilde{\mathbf{y}}[i,k,t] \cdot\mathbf{E}[i] - \sum_{j \neq k} \tilde{\mathbf{y}}[k,j,t] \cdot\mathbf{E}[k], \nonumber\\
\frac{\mathrm{d}\mathbf{I}[k]}{\mathrm{d}t} &= \bm{\sigma}[k]\cdot \mathbf{E}[k] - \bm{\gamma}[k] \cdot\mathbf{I}[k] \nonumber\\ &+ \sum_{i \neq k} \tilde{\mathbf{y}}[i,k,t]\cdot \mathbf{I}[i] - \sum_{j \neq k} \tilde{\mathbf{y}}[k,j,t] \cdot\mathbf{I}[k], \nonumber \\
\frac{\mathrm{d}\mathbf{R}[k]}{\mathrm{d}t} &= \bm{\gamma}[k] \cdot\mathbf{I}[k] + \sum_{i \neq k} \tilde{\mathbf{y}}[i,k,t]\cdot \mathbf{R}[i] - \sum_{j \neq k} \tilde{\mathbf{y}}[k,j,t] \cdot\mathbf{R}[k], \nonumber\\
\frac{\mathrm{d}\mathbf{V}[k]}{\mathrm{d}t} &= \frac{\mathbf{a}[k]}{T} + \sum_{i \neq k} \tilde{\mathbf{y}}[i,k,t]\cdot \mathbf{V}[i] - \sum_{j \neq k} \tilde{\mathbf{y}}[k,j,t] \cdot\mathbf{V}[k],
\end{align}
where $\bm{\beta}[k]$, $\bm{\gamma}[k]$, and $\bm{\sigma}[k]$ are hyper-parameter for region $k$. These hyperparameters are fitted on the dataset using maximum likelihood estimation. $\tilde{\mathbf{y}}$ is the normalized OD matrix.

Finally, the simulator will output the total number of newly infected people across all the regions and we aim to minimize this value.

\noindent\textbf{Solver at test time:} At test time, for a fair comparison, we use the same optimization solver for all the methods. Specifically, we use mirror descent \citep{beck2003mirror} so that the updated decision variable will still variable satisfy the constraints. Specifically, the update rule takes the following form at $t$-th iteration:
\begin{align}
    \mathbf{a}_{t+1}[i] = \text{Budget}\cdot\frac{\mathbf{a}_t[i] \exp(-\gamma\nabla_if(\mathbf{a}_t))}{\sum_{j=1}^n\mathbf{a}_t[i]\exp(-\gamma\nabla_j f(\mathbf{a}_t))},
\end{align}
where $\gamma$ is the learning rate. We set the learning rate as $0.01$ and repeat $500$ iterations. We empirically found that this solver solves this optimization problem very well.

\noindent\textbf{Model Hyperparameters:}
For the two-stage model, DFL, LODL and SO-EBM, the forecaster uses GMM with a different number of components and use 100 samples to estimate the expectation of the objective as we found that more samples bring little performance gain. The forecaster is a DC-RNN \citep{li2018diffusion1} which adopts an encoder-decoder architecture. The encoder and decoder both have two hidden layers with a hidden size of 128. The forecaster takes the OD matrices of last week as input features and predicts the OD matrices of next week. The forecaster outputs the mean, log variance, and weight for each GMM component. During training, we sample from the GMM using the Gumbel softmax trick \citep{jang2017categorical} to make the sampling process differentiable.  Since the decision variable is a simplex, we train SO-EBM with projected Langevin dynamics. Specifically, at each iteration of the Langevin dynamics, we project the decision variable into the simplex. The number of iterations of the Langevin dynamics is 100 and the step size is 0.05.

 For a fair comparison, \ours employs the same encoder as the DC-RNN architecture and uses 100 attention points.  During training, \ours samples 100 actions $\mathbf{a}$ uniformly from the constrained space, \ie~ the simplex, for each $(\mathbf{x},\mathbf{y})$ pair at each iteration for function fitting. To uniformly sample from the simplex, we sample from the Dirichlet distribution where all parameters are 1.

\noindent\textbf{Model Optimization:}
We use the Adam~\citep{kingma2015adam} algorithm for model optimization. The number of training epochs is 50. The learning rate for all the methods is $10^{-4}$. DFL, LODL and SO-EBM use the two-stage model as the pre-trained model for faster training convergence.

\subsection{Inventory Optimization}
\label{s:customer}

\noindent\textbf{Optimization objective}  In this task, a department store is tasked with predicting the sales $\mathbf{y} \in \mathbb{R}^7$ for the upcoming 7th-14th days based on the past 14 days' sales data $\mathbf{x} \in \mathbb{R}^{14}$ for a specific product, and accordingly, determining the best replenishment strategy $\mathbf{a} \in \mathbb{R}^7$ for each day. The optimization objective is a combination of an under-purchasing penalty, an over-purchasing penalty, and a squared loss between supplies and demands: 
\begin{align}
   \text{minimize}_{\mathbf{a}\in \mathbb{R}^7} \mathbb{E}_{p(\mathbf{y}|\mathbf{x})}&\sum_{i=1}^7[20(\mathbf{y}[i]-\mathbf{a}[i])_{+} + 5(\mathbf{a}[i]-\mathbf{y}[i])_{+} \nonumber\\ &+ (\mathbf{a}[i]-\mathbf{y}[i])^2]\nonumber \\
   &\text{subject to} \quad 0 \le \mathbf{a}[i] \le 3, \forall i, \nonumber
\end{align}
where $(v)_{+}$ denote $\text{max}\{v,0\}$.

\noindent\textbf{Solver at test time:} At test time, for a fair comparison, we use the same optimization solver for all the methods. Specifically, we use projected gradient descent and the gradient update step adopts the Adam~\citep{kingma2015adam} optimizer. The learning rate is $0.1$ and we repeat $500$ iterations. We empirically found that this solver solves this optimization problem very well.

\noindent\textbf{Model Hyperparameters:}
The forecaster of the two-stage model, DFL, LODL and SO-EBM uses a two-layer long short-term memory network (LSTM) \citep{hochreiter1997long1} as a feature extractor which is further stacked by a linear layer. The forecaster takes the historical item sales in the last 14 days as input features and outputs the forecasted item sales for the 7th to 14th days in the future. The network has a hidden size of 128. SO-EBM draws 512 samples from the proposal distribution to estimate the gradient of the model parameters. The proposal distribution is a mixture of Gaussians with 3 components where the variances are $\{0.05, 0.1, 0.2\}$.

For a fair comparison, \ours uses the same LSTM architecture as the encoder and 230 attention points. During training, the two-stage model, DFL, LODL and SO-EBM use 100 samples to estimate the expected objective as more samples provide little performance gain. 

\noindent\textbf{Model Optimization:}
We use the Adam~\citep{kingma2015adam} algorithm for model optimization. The number of training epochs is 200. The learning rate for all the methods is $10^{-3}$. DFL, LODL and SO-EBM use the two-stage model as the pre-trained model for faster training convergence.



\end{document}